\documentclass[journal]{IEEEtran}
\usepackage{amssymb,amsthm, amsmath}
\usepackage{tabularx}
\usepackage{multirow}
\usepackage{booktabs}
\usepackage{graphicx}
\usepackage{algorithm}
\usepackage{algorithmic}
\usepackage{color}

\newtheorem{theorem}{Theorem}

\newtheorem{proof of theorem}{Proof of Theorem}
\usepackage{threeparttable}
\usepackage{subfigure}
\usepackage{color}
\usepackage{float}
\usepackage{enumitem}
\usepackage{graphicx}
\usepackage{placeins}
\usepackage{capt-of}
\usepackage{amsmath}
\usepackage{booktabs}
\usepackage{multirow}
\usepackage{threeparttable}
\usepackage{xcolor}
\usepackage{etoolbox}

\makeatletter
\patchcmd{\@makecaption}
  {\\}
  {\\[4pt]}
  {}{}
\makeatother
\begin{document}
\flushbottom
\title{Learning Spatial-Spectral Refinement and Calibrating Complementary Observations for Hyperspectral Image Super-Resolution}

\author{Liqian Yang, Xingchi Chen, Xinfeng Gui, Xiangyong Cao, Qianxin Yi
	\thanks{Liqian Yang, Xingchi Chen and Qianxin Yi are with the School of Management, Zhengzhou University, Zhengzhou, 450001, China. Xinfeng Gui is with the Center for Intelligent Decision-making and Machine Learning, Xi'an Jiaotong University, Xi'an, 710049, China. Xiangyong Cao is with the School of Computer Science and Technology, Xi'an Jiaotong University, Xi'an, 710049, China.}
	\thanks{Qianxin Yi is the corresponding author. Email: yiqianxin01@163.com.}}


\maketitle

\begin{abstract}
Hyperspectral and multispectral image fusion (HMIF) aims to reconstruct a
high-resolution hyperspectral image (HR-HSI) by combining the fine spatial
details of a high-resolution multispectral image (HR-MSI) with the rich
spectral information of a low-resolution hyperspectral image (LR-HSI).
Recent advances in implicit neural representations (INRs) have enabled
flexible coordinate-based modeling for HMIF; however, existing INR-based
approaches may not fully capture fine-grained spatial structures and rich
spectral dependencies. Moreover, the LR-HSI and HR-MSI are primarily
incorporated through degradation-consistency constraints, leaving their
complementary information underexploited. To address these limitations, we
propose Two-Stage Reconstruction with Implicit Tensor Neural Representation
(TSR-ITNR), a unified self-supervised framework integrating representation
refinement and observation-guided calibration. In Stage 1, TSR-ITNR learns an
implicit Tucker representation and refines its low-rank spatial coefficient
tensor and spectral basis to better capture fine spatial structures and
interband correlations. A fixed pretrained denoiser further provides a deep
prior for the preliminary reconstruction. In Stage 2, parameter-free
calibration derives complementary and noninterfering corrections from both
observations to recover information insufficiently captured in Stage 1.
Theoretical analysis establishes the geometry-preserving property of spectral
refinement and the orthogonal complementarity of calibration. Extensive
experiments on multiple benchmark datasets demonstrate strong quantitative,
visual, and spectral reconstruction performance without ground-truth HR-HSI
supervision. Beyond conventional reconstruction metrics, we further assess
the effectiveness of TSR-ITNR using downstream semantic segmentation
accuracy.
\end{abstract}


\section{Introduction}
Hyperspectral images (HSIs) associate each spatial location with a densely sampled spectrum, thereby providing a three-dimensional description of the observed scene. Owing to this joint spatial-spectral characterization, HSIs have been widely used in remote sensing~\cite{camps2013advances},~\cite{manolakis2002detection},~\cite{cao2020hyperspectral}, geology~\cite{van2012multi}, medicine~\cite{pike2015minimum}, and food science~\cite{gowen2007hyperspectral}. However, the physical constraints of spectral imaging systems make high spatial and spectral resolutions difficult to achieve simultaneously~\cite{qu2018unsupervised},~\cite{dian2021recent}. Hyperspectral sensors acquire densely sampled spectral information but typically at limited spatial resolution, whereas multispectral images (MSIs) resolve finer spatial structures at the expense of spectral detail because they contain only a few broad bands~\cite{dian2021recent},~\cite{yokoya2017hyperspectral}. Consequently, neither modality alone contains both dense spectral information and fine spatial details. To overcome this tradeoff, hyperspectral and multispectral image fusion (HMIF) combines the rich spectral information of a low-resolution HSI (LR-HSI) with the fine spatial detail of a coregistered high-resolution MSI (HR-MSI), providing an effective solution for reconstructing a high-resolution HSI (HR-HSI) using existing imaging systems~\cite{yokoya2017hyperspectral},~\cite{yokoya2011coupled},~\cite{simoes2014convex}.

Existing HMIF approaches can be broadly categorized into three groups: low-rank representation-based methods, deep learning-based methods, and hybrid-based methods. Low-rank representation-based methods exploit spatial-spectral correlations by modeling the HR-HSI in a low-dimensional matrix or tensor subspace and incorporating handcrafted priors to regularize the reconstruction, thereby alleviating the ill-posedness of HMIF~\cite{liu2024asymptotic},~\cite{wei2015hyperspectral},~\cite{xue2021spatial},~\cite{xue2024tensor},~\cite{long2025parameter}. However, these predefined structural assumptions limit their adaptability across diverse scenes. Deep learning-based methods employ end-to-end or scene-specific networks to learn nonlinear spatial-spectral priors for HR-HSI reconstruction~\cite{dian2018deep},~\cite{chen2024cyclic},~\cite{li2025enhanced},~\cite{du2026unsupervised}. However, supervised approaches depend on costly paired training data, whereas unsupervised variants are often constrained by architectural bias and limited modeling of complex hyperspectral structures. Hybrid-based methods combine model-driven priors, such as low-rank structures and physical observation models, with deep neural networks, thereby benefiting from both structural interpretability and powerful representation learning~\cite{dian2020regularizing},~\cite{shen2021admm},~\cite{wang2023self},~\cite{shao2026hpgc}. However, these hybrid methods often incorporate structural models and neural priors through loosely coupled modules or alternating procedures, restricting their interaction and limiting the joint characterization of complex, scene-dependent spatial-spectral correlations.

Motivated by recent advances in implicit neural representations (INRs), several studies have integrated coordinate-based neural networks with low-rank tensor decompositions to construct continuous and compact representations of multidimensional signals. By replacing discrete decomposition factors with neural functions, these methods preserve explicit low-rank structures while benefiting from the nonlinear representation capability of INRs, and have achieved promising results in general multidimensional data representation and recovery~\cite{luo2023low},~\cite{cheng2026low},~\cite{vemuri2026f}. More recently, CLoRF extended continuous low-rank representations to self-supervised HSI-MSI fusion, demonstrating the potential of INR-based hybrid methods for HMIF~\cite{wang2025hyperspectral}. However, its underlying matrix factorization still unfolds the two spatial dimensions into a single mode, while the LR-HSI and HR-MSI are incorporated primarily through data-fidelity constraints. Consequently, jointly preserving the native high-order structure and fully exploiting the complementary spatial and spectral information remains challenging for INR-based hybrid methods.

To address the above limitations, we propose a novel self-supervised HMIF
framework, termed two-stage reconstruction with implicit tensor neural
representation (TSR-ITNR). TSR-ITNR first reconstructs a preliminary HR-HSI
by learning a compact implicit Tucker representation and jointly refining its
low-rank spatial coefficients and spectral basis under the data-consistency
constraints imposed by both observations. During this process, a fixed
pretrained denoiser is incorporated into the alternating optimization to
provide a deep prior. The resulting preliminary estimate is then calibrated
by exploiting the complementary information contained in the LR-HSI and
HR-MSI. The
main contributions are summarized as follows:
\begin{itemize}
    \item To balance compact global modeling and detail recovery while exploiting complementary observations, we develop a two-stage framework TSR-ITNR. Implicit Tucker representation and refinement (ITRR) integrates compact tensor modeling with spatial-spectral refinement, whereas complementary observation guided calibration (COGC) calibrates the initial reconstruction using complementary LR-HSI and HR-MSI information.

    \item We theoretically analyze attention-guided geometry-preserving spectral refiner (AGPSR) and COGC, establishing the basis independence and Gram preservation of AGPSR and the orthogonal complementarity and minimum-change property of COGC.

    \item Extensive experiments on simulated datasets demonstrate its superiority over representative methods and validate its key components, while downstream semantic segmentation experiments further confirm its practical utility.
\end{itemize}

The remainder of this article is organized as follows.
Section~\ref{sec:related_work} reviews related work on HMIF. We formulate the
HMIF problem and detail the two-stage TSR-ITNR framework in
Section~\ref{sec:proposed_method}. The theoretical properties of AGPSR and
COGC are analyzed in Section~\ref{sec:Theoretical Analysis}. Experimental
results are presented and discussed in Section~\ref{sec:experiments}, and
conclusions are drawn in Section~\ref{sec:conclusion}.

\section{RELATED WORK}
\label{sec:related_work}
\subsection{Low-Rank Representation-Based Methods}

Low-rank representation-based methods mitigate the ill-posedness of HMIF by exploiting the intrinsic low-dimensional structure of HR-HSIs and can be broadly divided into matrix- and tensor-based formulations. Matrix-based methods matricize the hyperspectral data cube and impose low-rank factorization to obtain compact representations of spatial-spectral correlations~\cite{zhang2016multispectral},~\cite{liu2020truncated}. In contrast, tensor-based methods preserve the native multiway structure of the data cube and capture correlations along its spatial and spectral modes through Tucker decomposition~\cite{dian2017hyperspectral},~\cite{li2018fusing}, canonical polyadic (CP) decomposition~\cite{kanatsoulis2018hyperspectral}, tensor singular value decomposition (t-SVD)~\cite{dian2019hyperspectral}, tensor train (TT) decomposition~\cite{dian2019learning}, and tensor ring (TR) decomposition~\cite{chen2022hyperspectral}. More recently, Xu et al.~\cite{xu2024coupled} coupled spatial and spectral tensor factors to strengthen two-dimensional spatial modeling, whereas Dian et al.~\cite{dian2024hyperspectral} extended tensor nuclear-norm regularization across multiple modes for more comprehensive correlation modeling. Despite their interpretability, these methods typically rely on handcrafted priors derived from specific data assumptions, which may generalize poorly to scenes with different spatial-spectral characteristics.

\subsection{Deep Learning-Based Methods}

To learn more expressive spatial-spectral priors, deep learning-based methods have been extensively investigated and can generally be divided into supervised and unsupervised paradigms. Supervised approaches learn end-to-end fusion mappings from paired HR-HSI, LR-HSI, and HR-MSI triplets. Among them, CNN-based methods employ hierarchical convolutions to capture local spatial-spectral features~\cite{dian2018deep},~\cite{zhang2020ssr}, whereas Transformer-based methods introduce attention mechanisms to model long-range dependencies and global interactions,~\cite{chen2024cyclic},~\cite{deng2023psrt}. Reciprocal Transformer further promotes bidirectional information exchange between the LR-HSI and HR-MSI~\cite{ma2024reciprocal}. However, supervised methods are constrained by demanding paired-data requirements and high training costs. To overcome these limitations, unsupervised methods learn scene-specific priors directly from the observed image pair without HR-HSI supervision~\cite{li2025enhanced},~\cite{uezato2020guided}. In particular, Fang et al.~\cite{fang2024cs2dips} coupled spatial and spectral DIPs to extract complementary features in parallel from the HR-MSI and LR-HSI. Although unsupervised methods eliminate the need for paired training data, their network-induced priors may not sufficiently capture the global low-rank spatial-spectral correlations of HSIs, thereby limiting reconstruction performance.

\subsection{Hybrid-Based Methods}

Hybrid-based methods combine low-rank representation-based methods with deep learning by embedding low-rank subspace or decomposition models into neural reconstruction frameworks. Explicit low-rank modeling captures global spatial-spectral correlations and constrains the solution space, while neural networks provide flexible learned priors for estimating or refining the decomposition factors~\cite{dian2020regularizing},~\cite{shen2021admm},~\cite{xie2020mhf},~\cite{li2024slrcnn},~\cite{chen2025learning}. Wang et al.~\cite{wang2025hyperspectral} represented the latent HR-HSI through continuous low-rank matrix factorization parameterized by spatial and spectral INRs.  This coupling provides a structurally constrained yet flexible framework for HMIF. However, balancing the compact global modeling of this continuous low-rank formulation with the flexible representation of fine-grained and spatially varying details remains challenging. Moreover, merely treating the degraded observations as network inputs or incorporating them via data-fidelity terms may be insufficient to fully exploit their complementary information and may introduce mutual interference during reconstruction.

To address these limitations, we propose TSR-ITNR, a self-supervised
two-stage framework for HMIF that unifies compact structural modeling,
flexible neural refinement, and explicit exploitation of observation
complementarity. In the first stage, ITRR employs an implicit Tucker
representation to preserve the native high-order structure and capture global
low-rank spatial-spectral correlations, while dedicated spatial and spectral
refiners enhance fine-grained spatial details and spectral information beyond
the compact factorization. This stage is optimized directly from the observed
LR-HSI and HR-MSI without paired HR-HSI training data and incorporates a fixed
pretrained denoiser as a deep prior, reducing reliance on handcrafted
regularizers. In the second stage, COGC derives complementary corrections from
the LR-HSI and HR-MSI to remove residual inconsistencies in the preliminary
reconstruction without mutual interference. Together, these designs combine
the structural interpretability of low-rank tensor modeling with neural
flexibility and more effective use of complementary observations.

\begin{figure*}[!t]
  \centering
  \includegraphics[width=1.0\textwidth]{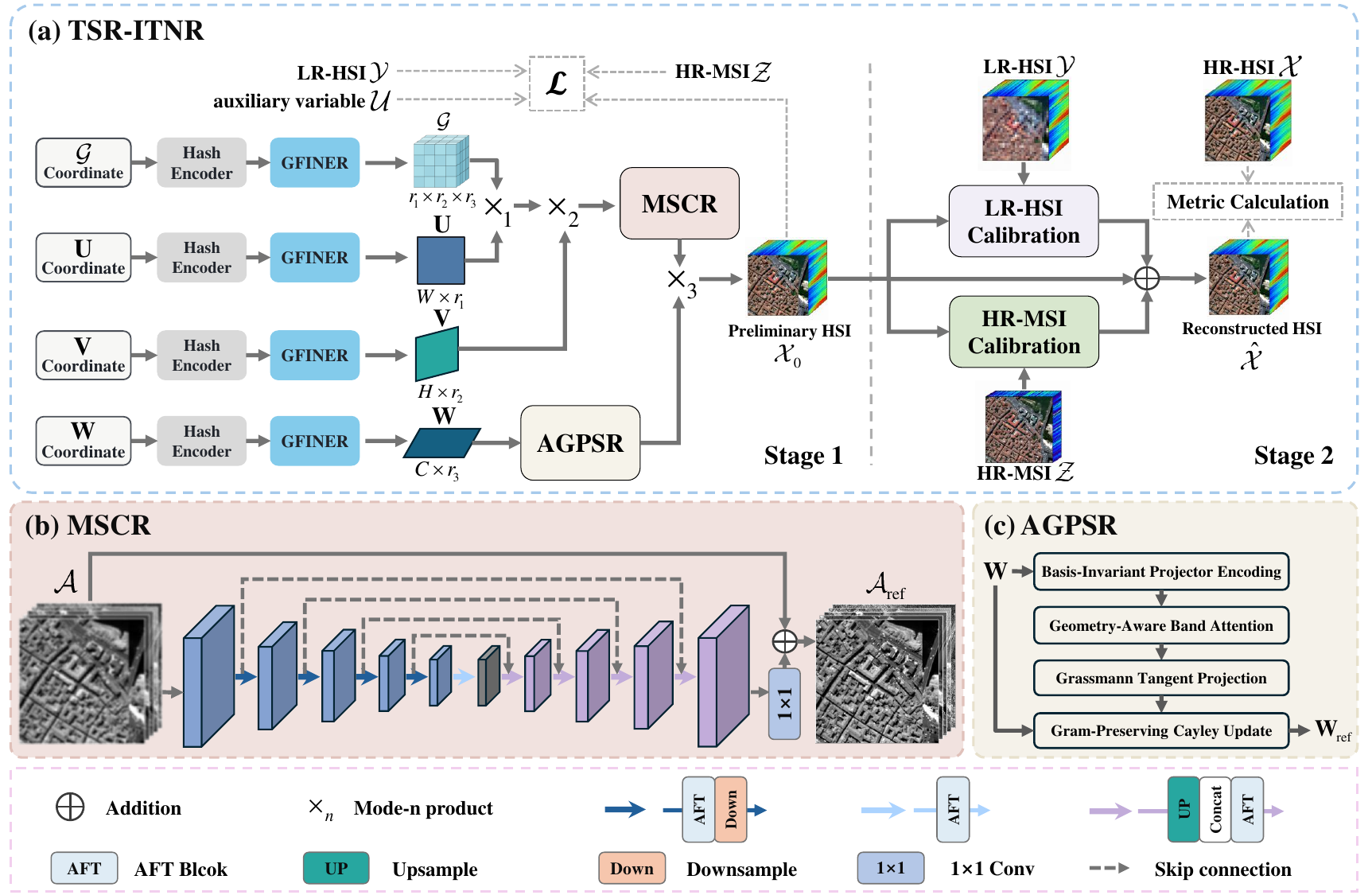}
  \caption{Overview of the proposed Two-Stage Reconstruction with implicit tensor neural representation (TSR-ITNR) framework and its key refinement modules. (a) Overall framework of TSR-ITNR, comprising two tightly coupled and complementary stages: implicit Tucker representation and refinement (ITRR) and complementary observation-guided calibration (COGC). (b) Architecture of the multi-scale spatial coefficient refiner (MSCR). (c) Architecture of the attention-guided geometry-preserving spectral refiner (AGPSR).}
  \label{fig:framework}
\end{figure*}

\section{METHODOLOGY}
\label{sec:proposed_method}
\subsection{Problem Formulation}
\subsubsection{Degradation Model}
The HMIF task aims to reconstruct an HR-HSI by fusing an LR-HSI and an HR-MSI. Let $\mathcal{Y} \in \mathbb{R}^{w \times h \times C}$ represents the LR-HSI, $\mathcal{Z} \in \mathbb{R}^{W \times H \times c}$ represents the HR-MSI, and $\mathcal{X} \in \mathbb{R}^{W \times H \times C}$ represents the HR-HSI, where $W \gg w,\ H \gg h,\ \text{and}\ C \gg c$. The LR-HSI and HR-MSI are obtained by applying spatial and spectral degradation to the HR-HSI, respectively, and their degradation models are given as:
\begin{equation}
\mathcal{Y}=\mathcal{X}\times_1\mathbf{P}_{1}\times_2\mathbf{P}_{2},\quad\mathcal{Z}=\mathcal{X}\times_3\mathbf{P}_{3},
\label{equation:1}
\end{equation}
where $\mathbf{P}_{1}\in\mathbb{R}^{w \times W}$ and $\mathbf{P}_{2} \in\mathbb{R}^{h \times H}$ are the degradation operators for blurring and downsampling in spatial width and height modes, respectively, and $\mathbf{P}_{3}\in\mathbb{R}^{c \times C}$ represents the spectral degradation matrix.

\subsubsection{Tucker-Based Low-Rank Reconstruction Formulation}
Directly recovering the high-dimensional HR-HSI from the two degraded observations is inherently ill-posed. To constrain the solution space and exploit the intrinsic multilinear low-rank structure of the latent HR-HSI, we employ a Tucker decomposition to construct a compact spatial-spectral representation:
\begin{equation}
\mathcal{X}=\mathcal{G}\times_1\mathbf{U}\times_2\mathbf{V}\times_3\mathbf{W}=\mathcal{A}\times_3\mathbf{W}\label{equation:2},
\end{equation}
where $\mathcal{G} \in \mathbb{R}^{r_1 \times r_2 \times r_3}$ represents the core tensor, while $\mathbf{U} \in \mathbb{R}^{W \times r_1}(r_1 \ll W)$, $\mathbf{V} \in \mathbb{R}^{H \times r_2}(r_2 \ll H)$ and $\mathbf{W} \in \mathbb{R}^{C \times r_3}(r_3 \ll C)$ denote the factor matrices associated with the two spatial modes and the spectral mode, respectively. In particular, $\mathcal{G}$, $\mathbf{U}$ and $\mathbf{V}$ jointly characterize the spatial structure of the latent HR-HSI and are combined to form the spatially low-rank coefficient tensor $\mathcal{A} \in \mathbb{R}^{W \times H \times r_3}$, as defined in (\ref{equation:2}). Correspondingly, $\mathbf{W}$ serves as the spectral basis matrix that maps the spatially low-rank coefficient tensor $\mathcal{A}$ from the latent spectral space to the original spectral domain.

Based on this representation, the recovery of the high-dimensional HR-HSI is reformulated as the joint estimation of the spatially low-rank coefficient tensor $\mathcal{A}$ and the spectral basis matrix $\mathbf{W}$ from the observed LR-HSI $\mathcal{Y}$ and HR-MSI $\mathcal{Z}$. Substituting (\ref{equation:2}) into the degradation models yields the following regularized optimization problem:
\begin{align}
& \underset{\mathcal{A},\mathbf{W}}{\arg\min}\;
\big\|
\mathcal{Y}
-
\mathcal{A}\times_3\mathbf{W}
\times_1 \mathbf{P}_1
\times_2 \mathbf{P}_2
\big\|_F^2+
\big\|
\mathcal{Z}
-
\mathcal{A}
\notag\\
& \quad\times_3\mathbf{W}
\times_3 \mathbf{P}_3
\big\|_F^2
+\lambda \mathcal{R}
\left(
\mathcal{A}\times_3\mathbf{W}
\right).
\label{equation:3}
\end{align}
The first and second data-fidelity terms enforce consistency of the reconstructed HR-HSI with the observed LR-HSI and HR-MSI, respectively, while $\mathcal{R}(\cdot)$ incorporates prior information to regularize the reconstruction.

\subsection{Overview of Framework}
As shown in Fig.~\ref{fig:framework}(a), TSR-ITNR organizes the reconstruction into structured representation refinement and observation calibration, implemented by ITRR and COGC, respectively. ITRR constructs and refines the implicit Tucker representation of the latent HR-HSI, whereas COGC resolves the remaining reconstruction errors through complementary corrections derived from LR-HSI and HR-MSI. This division enables representation recovery and observation consistency to be addressed in a coordinated manner. The detailed formulations of the two stages are presented in the following subsections.

\subsection{Stage 1: Implicit Tucker Representation and Refinement}

Recent studies have demonstrated that combining INRs with functional tensor decomposition enables compact and efficient representation of high-dimensional data~\cite{vemuri2026f},~\cite{zhou2026efficient}. Motivated by this progress, we develop ITRR as the first stage of TSR-ITNR. ITRR employs an implicit Tucker representation to compactly model the latent HR-HSI and refines its spatial coefficients and spectral basis to recover fine spatial details and further exploit spectral information. The reconstruction is then solved via plug-and-play half-quadratic splitting (PnP-HQS), alternating representation refinement with a fixed pretrained denoiser as a deep spatial prior. Upon convergence, ITRR yields a preliminary HR-HSI for subsequent calibration.

Specifically, four independent branches, each consisting of a hash encoder~\cite{muller2022instant} followed by a GFINER network~\cite{zhu2026finer++}, generate the Tucker core tensor and three factor matrices from their respective coordinate sets as follows:
\begin{align}
\mathcal{G}
&=
\mathcal{F}_{\mathcal{G},\Theta_1}
\left(
\mathcal{H}_{\Phi_1}
\left(\mathbf{C}_{\mathcal{G}}\right)
\right),
\notag\\
\mathbf{U}
&=
\mathcal{F}_{\mathbf{U},\Theta_2}
\left(
\mathcal{H}_{\Phi_2}
\left(\mathbf{C}_{\mathbf{U}}\right)
\right),
\notag\\
\mathbf{V}
&=
\mathcal{F}_{\mathbf{V},\Theta_3}
\left(
\mathcal{H}_{\Phi_3}
\left(\mathbf{C}_{\mathbf{V}}\right)
\right),
\notag\\
\mathbf{W}
&=
\mathcal{F}_{\mathbf{W},\Theta_4}
\left(
\mathcal{H}_{\Phi_4}
\left(\mathbf{C}_{\mathbf{W}}\right)
\right),
\label{equation:4}
\end{align}
where $\mathbf{C}_{\mathcal{G}}$, $\mathbf{C}_{\mathbf{U}}$, $\mathbf{C}_{\mathbf{V}}$, and $\mathbf{C}_{\mathbf{W}}$ denote the
coordinate sets of $\mathcal{G}$, $\mathbf{U}$, $\mathbf{V}$, and
$\mathbf{W}$, respectively. For $i=1,\ldots,4$,
$\mathcal{H}_{\Phi_i}$ and $\mathcal{F}_{\Theta_i}$ denote the
multiresolution hash encoder and GFINER network in the $i$th branch,
parameterized by $\Phi_i$ and $\Theta_i$, respectively. 

The outputs $\mathcal{G}$, $\mathbf{U}$, $\mathbf{V}$, and $\mathbf{W}$
represent the Tucker core tensor, two spatial factor matrices, and spectral
factor matrix, respectively. The Tucker core and spatial factors are combined
via mode-$1$ and mode-$2$ products to form the low-rank spatial coefficient
tensor $\mathcal{A}$, which is then refined by the multi-scale spatial
coefficient refiner (MSCR). In parallel, AGPSR refines $\mathbf{W}$ to better
capture spectral correlations. Finally, the refined spatial and spectral
components are combined to reconstruct the preliminary HR-HSI:
\begin{align}
\mathcal{X}_0
&=
\operatorname{MSCR}
\left(
\mathcal{A}
\right)\times_3
\operatorname{AGPSR}
\left(
\mathbf{W}
\right)=\mathcal{A}_{\mathrm{ref}}\times_3\mathbf{W}_{\mathrm{ref}},
\label{equation:5}
\end{align}
where $\mathcal{A}_{\mathrm{ref}}=\operatorname{MSCR}(\mathcal{A})$
and $\mathbf{W}_{\mathrm{ref}}=\operatorname{AGPSR}(\mathbf{W})$
denote the refined spatial coefficient tensor and spectral basis matrix,
respectively.

\subsubsection{Multi-scale Spatial Coefficient Refiner}

Although $\mathcal{A}$ captures compact global spatial structure through Tucker decomposition, the resulting multilinear representation may lack sufficient flexibility to model fine-grained and spatially varying details. To address this limitation, we introduce an MSCR to adaptively enhance the spatial representation of $\mathcal{A}$ while preserving its intrinsic low-rank structure.

As illustrated in Fig.~\ref{fig:framework}(b), inspired by the hierarchical
encoder-decoder paradigm and skip connections of U-Net~\cite{ronneberger2015u},
we design a four-level multiscale architecture tailored to spatial coefficient
refinement. Along the encoding path, adaptive feature transformation (AFT)
blocks extract spatial features, while successive downsampling enlarges the
receptive field to aggregate contextual information at multiple scales. Along
the decoding path, the features are progressively upsampled and fused with
their encoder counterparts through skip connections, followed by AFT blocks
to integrate multiscale context and enhance fine spatial structures. Finally,
a $1\times1$ convolution maps the decoded features back to the coefficient
space and predicts a residual correction $\Delta\mathcal{A}$. The refined
spatial coefficient tensor is formulated as:
\begin{equation}
\mathcal{A}_{\mathrm{ref}}
=
\operatorname{MSCR}(\mathcal{A})
=
\mathcal{A}+\Delta\mathcal{A}.
\label{equation:mscr}
\end{equation}

The AFT block and its embedded CAM are illustrated in
Fig.~\ref{fig:AFT and CAM}(a) and (b), respectively. Given an input feature
$\mathbf{F}$, the main branch of AFT applies two $3\times3$ convolutions,
each followed by group normalization, with a ReLU activation inserted after
the first convolutional layer. The resulting features are recalibrated by
channel attention mechanism (CAM) and added to a $1\times1$ convolutional shortcut, followed by ReLU
activation. Within CAM, channel weights are generated through pooling, two
fully connected layers with an intermediate ReLU, and a sigmoid function,
and are then applied to the input features by element-wise multiplication.
Embedded at multiple scales in MSCR, AFT enhances the spatial coefficient
tensor $\mathcal{A}$ with contextual information and fine structural details.

\begin{figure}[!tp]
  \centering
  \includegraphics[width=0.4\textwidth]{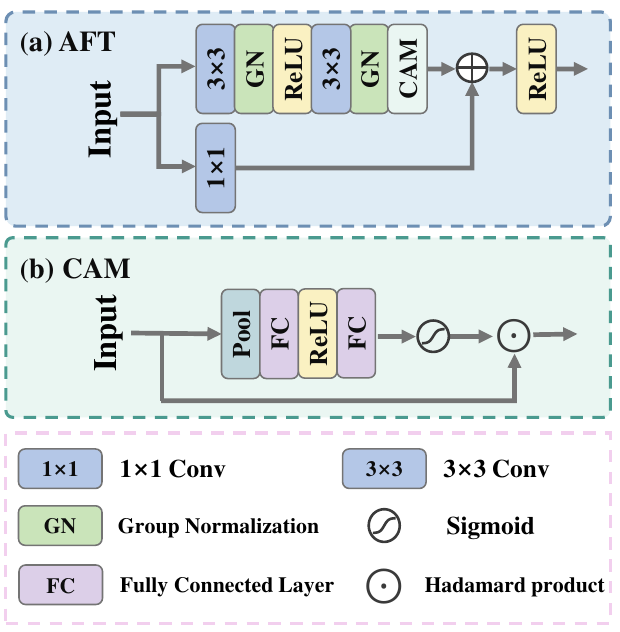}
  \caption{Detailed architectures of the adaptive feature transformation (AFT) block employed in MSCR and its embedded channel attention mechanism (CAM): (a) AFT block and (b) CAM block.}\label{fig:AFT and CAM}
\end{figure}
\subsubsection{Attention-guided Geometry-preserving Spectral Refiner}

The purpose of AGPSR is to enhance spectral dependency modeling without
making the refinement sensitive to the arbitrary basis used to represent the
spectral subspace or distorting the geometry of the spectral factor. Assuming
that $\mathbf{W}\in\mathbb{R}^{C\times r_3}$ has full column rank, we compute
a thin QR factorization $\mathbf{W}=\mathbf{Q}\mathbf{R}$ and construct
$\boldsymbol{\Pi}_{\mathbf{W}}
=
\mathbf{Q}\mathbf{Q}^{\top}$.

Since the orthonormal basis $\mathbf{Q}$ is not unique, constructing band
tokens directly from it may introduce basis ambiguity. AGPSR therefore derives
the tokens from the rows of the basis-invariant projector
$\boldsymbol{\Pi}_{\mathbf{W}}$, ensuring that the attention weights are
determined by the underlying spectral-subspace structure rather than by a
particular basis representation.

Let $\mathbf{Z}$ denote the band-token matrix formed by jointly encoding the
rows of $\boldsymbol{\Pi}_{\mathbf{W}}$ and their corresponding fixed spectral
positional embeddings. The geometry-aware attention matrix of the $h$th head is computed as
\begin{equation}
\mathbf{A}^{(h)}
=
\operatorname{softmax}\!\left(
\frac{
\mathbf{Z}_{q}^{(h)}
\bigl(\mathbf{Z}_{k}^{(h)}\bigr)^{\top}
}{
\sqrt{d_h}
}
-\beta_h\mathbf{D}^{\mathrm{sp}}
-\gamma_h\mathbf{D}^{\mathrm{geo}}
\right),
\label{eq:agpsr-band-attention}
\end{equation}
where $\mathbf{Z}_{q}^{(h)}$ and $\mathbf{Z}_{k}^{(h)}$ denote the query and
key representations of the $h$th head, respectively, and $d_h$ is their
feature dimension. The matrices $\mathbf{D}^{\mathrm{sp}}$ and
$\mathbf{D}^{\mathrm{geo}}$ contain pairwise distances between spectral
positions and between the corresponding projector rows, respectively. The
head-specific nonnegative parameters $\beta_h$ and $\gamma_h$ control the
strengths of these distance biases. Consequently, different
heads can capture complementary inter-band dependencies. Their responses are
then aggregated to form the attention-guided proposal:
\begin{equation}
\mathbf{S}
=
\sum_{h=1}^{N_{\mathrm{att}}}
\rho_h\mathbf{A}^{(h)}\mathbf{Q},
\label{eq:agpsr-proposal}
\end{equation}
where $N_{\mathrm{att}}$ denotes the number of attention heads and $\rho_h$
is the normalized fusion weight of the $h$th head. Since the attention
matrices are constructed from the basis-invariant projector, the learned
inter-band attention patterns are independent of any particular orthonormal
basis representation. The resulting $\mathbf{S}$ serves as an
attention-guided proposal for the subsequent Grassmann tangent projection.

The attention-guided proposal $\mathbf{S}$ may contain an in-subspace
component that only changes the basis representation without refining the
spectral subspace. To isolate the effective subspace-changing direction, we
project $\mathbf{S}$ onto the orthogonal complement of the current subspace:
\begin{equation}
\mathbf{T}
=
\left(
\mathbf{I}_C-\boldsymbol{\Pi}_{\mathbf{W}}
\right)\mathbf{S}.
\end{equation}
The resulting $\mathbf{T}$ satisfies
$\mathbf{Q}^{\top}\mathbf{T}=\mathbf{0}$ and defines a valid Grassmann tangent
direction. However, directly applying an additive update along this direction
may alter the scales and mutual inner products of the spectral basis vectors.
We therefore convert $\mathbf{T}$ into a skew-symmetric generator:
\begin{equation}
\mathbf{K}
=
\mathbf{T}\mathbf{Q}^{\top}
-
\mathbf{Q}\mathbf{T}^{\top},
\label{eq:agpsr-generator}
\end{equation}
and update the spectral factor through the Cayley transform:
\begin{equation}
\mathbf{C}_{\alpha}
=
\left(
\mathbf{I}_{C}-\frac{\alpha}{2}\mathbf{K}
\right)^{-1}
\left(
\mathbf{I}_{C}+\frac{\alpha}{2}\mathbf{K}
\right),
\;
\mathbf{W}_{\mathrm{ref}}
=
\mathbf{C}_{\alpha}\mathbf{W},
\label{eq:agpsr-cayley}
\end{equation}
where $\alpha$ controls the update magnitude. Since
$\mathbf{K}^{\top}=-\mathbf{K}$, the Cayley transform
$\mathbf{C}_{\alpha}$ is orthogonal, and therefore
\begin{equation}
\mathbf{W}_{\mathrm{ref}}^{\top}\mathbf{W}_{\mathrm{ref}}
=
\mathbf{W}^{\top}\mathbf{W}.
\end{equation}
Consequently, AGPSR converts the attention-guided proposal into a Grassmann
tangent update of the spectral subspace. As established in
Theorem~\ref{thm:agpsr}, the resulting refinement is independent of the
orthonormal basis chosen for $\operatorname{range}(\mathbf{W})$ and exactly
preserves $\mathbf{W}^{\top}\mathbf{W}$, and hence all column inner products
and singular values of $\mathbf{W}$.

\subsubsection{Optimization via PnP-HQS}

Following the spatial and spectral refinements in \eqref{equation:5}, we
replace $\mathcal{A}$ and $\mathbf{W}$ in the general HMIF formulation
\eqref{equation:3} with $\mathcal{A}_{\mathrm{ref}}$ and
$\mathbf{W}_{\mathrm{ref}}$, respectively. The resulting optimization problem
for preliminary HR-HSI reconstruction is formulated as:
\begin{align}
&\underset{\mathcal{A}_{\mathrm{ref}},\,\mathbf{W}_{\mathrm{ref}}}
{\operatorname{arg\,min}}\;
\big\|
\mathcal{Y}
-\mathcal{A}_{\mathrm{ref}}
\times_3\mathbf{W}_{\mathrm{ref}}
\times_1\mathbf{P}_1
\times_2\mathbf{P}_2
\big\|_F^2
+
\big\|\mathcal{Z}-
\notag\\[-1mm]
&\qquad
\mathcal{A}_{\mathrm{ref}}
\times_3\mathbf{W}_{\mathrm{ref}}
\times_3\mathbf{P}_3
\big\|_F^2
+
\lambda\mathcal{R}\!\left(
\mathcal{A}_{\mathrm{ref}}
\times_3\mathbf{W}_{\mathrm{ref}}
\right).
\label{equation:6}
\end{align}
We solve \eqref{equation:6} via PnP-HQS, combining the observation-adaptive
representation learned by the self-supervised network with a deep prior from
a fixed pretrained denoiser. Let $\mathcal{X}_0$ denote the preliminary
HR-HSI in \eqref{equation:5}. Introducing an auxiliary variable $\mathcal{U}$
to decouple data fidelity and regularization yields the equivalent constrained
problem:
\begin{align}
&\underset{\mathcal{X}_0,\,\mathcal{U}}
{\operatorname{arg\,min}}\;
\big\|
\mathcal{Y}
-\mathcal{X}_0
\times_1\mathbf{P}_1
\times_2\mathbf{P}_2
\big\|_F^2
+
\big\|
\mathcal{Z}
-\mathcal{X}_0
\times_3
\notag\\[-1mm]
&\qquad
\mathbf{P}_3
\big\|_F^2
+
\lambda\mathcal{R}(\mathcal{U}),
\qquad
\text{s.t.}\quad
\mathcal{U}=\mathcal{X}_0.
\label{equation:8}
\end{align}

The constrained problem (\ref{equation:8}) can be transformed into the following problem:
\begin{align}
&\underset{\mathcal{X}_0,\,\mathcal{U}}
{\operatorname{arg\,min}}\;
\big\|
\mathcal{Y}
-\mathcal{X}_0
\times_1\mathbf{P}_1
\times_2\mathbf{P}_2
\big\|_F^2
+
\big\|
\mathcal{Z}
-\mathcal{X}_0
\times_3
\notag\\[-1mm]
&\qquad
\mathbf{P}_3
\big\|_F^2
+
\lambda\mathcal{R}(\mathcal{U}) +
\frac{\beta}{2}
\big\|
\mathcal{U}-\mathcal{X}_0
\big\|_F^2,
\label{equation:9}
\end{align}
where $\beta>0$ is the penalty parameter. Problem \eqref{equation:9} is then
solved by alternating between the $\mathcal{X}_0$- and
$\mathcal{U}$-subproblems.
\paragraph{$\mathcal{X}_0$-subproblem}
The optimization of $\mathcal{X}_0$ is formulated as follows:
\begin{align}
&\underset{\mathcal{X}_0}
{\operatorname{arg\,min}}\;
\big\|
\mathcal{Y}
-\mathcal{X}_0
\times_1\mathbf{P}_1
\times_2\mathbf{P}_2
\big\|_F^2
+
\big\|
\mathcal{Z}
-\mathcal{X}_0
\times_3
\notag\\[-1mm]
&\qquad
\mathbf{P}_3
\big\|_F^2
+
\frac{\beta}{2}
\big\|
\mathcal{U}-\mathcal{X}_0
\big\|_F^2.
\label{equation:10}
\end{align}
We employ Adam optimizer~\cite{kingma2014adam} to solve $\mathcal{X}_0$-subproblem, and all three terms in (\ref{equation:10}) are used as the loss function for ITRR. 

\paragraph{$\mathcal{U}$-subproblem}
The optimization of $\mathcal{U}$ subproblem is formulated as follows:
\begin{align}
&\underset{\mathcal{U}}
{\operatorname{arg\,min}}\;
\lambda\mathcal{R}(\mathcal{U}) +
\frac{\beta}{2}
\big\|
\mathcal{U}-\mathcal{X}_0
\big\|_F^2.
\label{equation:11}
\end{align}
The $\mathcal{U}$ subproblem in \eqref{equation:11} can be viewed as Gaussian
denoising with a noise level of $\sigma=\sqrt{\lambda/\beta}$. Handcrafted
priors rely on predefined assumptions and lack the flexibility to characterize
complex and spatially varying image structures, potentially limiting
fine-detail recovery. We therefore employ a pretrained
DRUNet~\cite{zhang2021plug} as the PnP denoiser to introduce a more expressive
deep spatial prior.

To accommodate variations in the dynamic ranges of different spectral bands,
each band of the current HR-HSI estimate is first normalized to $[0,1]$.
Band-wise denoising is then performed using a fixed pre-trained DRUNet, with
the denoising strength adjusted according to the normalization scale of each
band. Finally, the denoised bands are restored to their original ranges and
reassembled to update the auxiliary variable $\mathcal{U}$.

We summarize the optimization process in Algorithm \ref{alg:HMIF}.

\begin{algorithm}[H]
\caption{ITRR-Based HMIF}
\label{alg:HMIF}
{\leftskip=1.3em\noindent\textbf{Require:} LR-HSI $\mathcal{Y}$, HR-MSI $\mathcal{Z}$, degradation operators $\mathbf{P}_{1}$, $\mathbf{P}_{2}$ and $\mathbf{P}_{3}$, regularization weight $\lambda$, penalty parameter $\beta$, and penalty growth factor $\mu$.\par}
\begin{enumerate}[label=\arabic*:, leftmargin=0pt, labelindent=0pt, itemindent=*]
\item \textbf{while} not converged \textbf{do}
\item \hspace{2em}$k = k + 1$.
\item \hspace{2em}Update $\mathcal{X}_0$ via ITRR in (\ref{equation:10}).
\item \hspace{2em}Update $\mathcal{U}$ via (\ref{equation:11}).
\item \hspace{2em}Update $\beta^{(k+1)} = \mu * \beta^{(k)}$.
\item \hspace{2em}\begin{minipage}[t]{\dimexpr\linewidth-2em\relax}
    Check the convergence condition: \scalebox{0.85}{$\tfrac{\|\mathcal{X}_0^{(k)} - \mathcal{X}_0^{(k-1)}\|_F}{\|\mathcal{X}_0^{(k-1)}\|_F} \leq \varepsilon$} \\
    and $k < k_{\max}$.
    \end{minipage}
\item \textbf{end while}
\end{enumerate}
{\leftskip=1.3em\noindent\textbf{Ensure:} The preliminary reconstructed HR-HSI $\mathcal{X}_0$.\par}
\end{algorithm}
\subsection{Stage 2: Complementary Observation-Guided Calibration}

Although ITRR yields a preliminary HR-HSI $\mathcal{X}_0$, the complementary
residual information in the two observations may not be fully recovered.
As illustrated in Stage 2 of Fig.~\ref{fig:framework}(a), COGC retains
$\mathcal{X}_0$ through a direct shortcut and converts the residuals of the
LR-HSI $\mathcal{Y}$ and HR-MSI $\mathcal{Z}$ into two targeted corrections
using the known degradation operators.

Let $\dagger$ denote the Moore--Penrose pseudoinverse. For compact notation,
we define
\begin{equation}
\boldsymbol{\Pi}_i
=
\mathbf{P}_i^{\dagger}\mathbf{P}_i,
\quad i=1,2,3,
\qquad
\mathbf{N}_3
=
\mathbf{I}_C-\boldsymbol{\Pi}_3,
\label{eq:cogc-projectors}
\end{equation}
where $\boldsymbol{\Pi}_i$ retains the components observable through
$\mathbf{P}_i$, while $\mathbf{N}_3$ extracts the spectral components
unobservable to the HR-MSI. Given $\mathcal{X}_0$, the residuals in the two
observation domains are defined as
\begin{equation}
\mathcal{E}_Y
=
\mathcal{Y}
-\mathcal{X}_0
\times_1\mathbf{P}_1
\times_2\mathbf{P}_2,
\quad
\mathcal{E}_Z
=
\mathcal{Z}
-\mathcal{X}_0
\times_3\mathbf{P}_3.
\label{eq:cogc-residuals}
\end{equation}

The HR-MSI branch lifts $\mathcal{E}_Z$ to the hyperspectral domain through
$\mathbf{P}_3^{\dagger}$ to correct spatially detailed components supported
by the HR-MSI. In parallel, the LR-HSI branch spatially back-projects
$\mathcal{E}_Y$ through $\mathbf{P}_1^{\dagger}$ and
$\mathbf{P}_2^{\dagger}$, and then restricts the correction using
$\mathbf{N}_3$ to complement the spectral information unavailable in the
HR-MSI. The two corrections are formulated as
\begin{equation}
\Delta\mathcal{X}_Y
=
\mathcal{E}_Y
\times_1\mathbf{P}_1^{\dagger}
\times_2\mathbf{P}_2^{\dagger}
\times_3\mathbf{N}_3,
\quad
\Delta\mathcal{X}_Z
=
\mathcal{E}_Z
\times_3\mathbf{P}_3^{\dagger}.
\label{eq:cogc-corrections}
\end{equation}

The final HR-HSI is obtained by adding both corrections to the direct-path
estimate:
\begin{equation}
\widehat{\mathcal{X}}
=
\mathcal{X}_0
+\Delta\mathcal{X}_Y
+\Delta\mathcal{X}_Z.
\label{eq:cogc-T10}
\end{equation}
By assigning the two observation-driven corrections to complementary
spectral components, COGC reduces redundant or interfering calibration.
Its complementarity and observation-consistency properties are analyzed in
Theorem~\ref{property of COGC}. Since COGC contains no learnable parameters,
it is applied after ITRR convergence without participating in the Stage-1
loss or backpropagation.

\section{Theoretical Analysis}
\label{sec:Theoretical Analysis}
\subsection{Basis Independence and Gram Preservation of AGPSR}
\label{sec:agpsr-theory}
Since a spectral subspace admits multiple orthonormal bases, AGPSR should be
independent of the particular basis representation while preserving the
intrinsic geometry of $\mathbf{W}$. The following theorem establishes these
properties.
\begin{theorem}[Basis Invariance and Gram Preservation]
\label{thm:agpsr}
\leavevmode\par\indent
For the AGPSR update defined in
\eqref{eq:agpsr-band-attention}--\eqref{eq:agpsr-cayley},
the following statements hold.
\begin{enumerate}

\item[(i)]
\(\mathbf{Q}^{\top}\mathbf{T}=\mathbf{0}\) and
\(\mathbf{K}\mathbf{Q}=\mathbf{T}\). Hence \(\mathbf{T}\) is a horizontal
tangent direction of the Grassmann manifold at
\(\operatorname{range}(\mathbf{Q})\).

\item[(ii)]
For every orthogonal \(\mathbf{O}\), replacing
\((\mathbf{Q},\mathbf{R})\) by
\((\mathbf{Q}\mathbf{O},\mathbf{O}^{\top}\mathbf{R})\) leaves
\(\boldsymbol{\Pi}_{\mathbf{W}}\), \(\mathbf{K}\),
\(\mathbf{C}_{\alpha}\), and \(\mathbf{W}_{\mathrm{ref}}\) unchanged.
Thus the update is independent of the orthonormal basis used to represent
the spectral subspace.

\item[(iii)]
The two matrices in the Cayley transform are nonsingular for every real
\(\alpha\), and \(\mathbf{C}_{\alpha}\) is orthogonal. Therefore
\begin{equation}
\mathbf{W}_{\mathrm{ref}}^{\top}\mathbf{W}_{\mathrm{ref}}
=
\mathbf{W}^{\top}\mathbf{W}.
\label{eq:agpsr-T5}
\end{equation}
Consequently, all column inner products, singular values, rank, spectral
norm, Frobenius norm, and the two-norm condition number of \(\mathbf{W}\)
are preserved.

\end{enumerate}
\end{theorem}
Theorem~\ref{thm:agpsr} characterizes three key properties of the AGPSR
update. First, $\mathbf{T}$ is a valid horizontal Grassmann tangent direction,
ensuring that AGPSR refines the spectral subspace rather than merely
reparameterizing its basis. Second, the update remains invariant to the choice
of orthonormal basis and therefore depends only on the represented subspace.
Finally, the Cayley transform yields an orthogonal update that exactly
preserves $\mathbf{W}^{\top}\mathbf{W}$ and the associated geometric
properties of $\mathbf{W}$. Overall, AGPSR provides an attention-guided
spectral refinement that is both basis-independent and geometry-preserving.
\subsection{Orthogonal Complementarity and Minimum-Change Property of COGC}
COGC is designed to exploit the two observations through complementary rather
than overlapping corrections. Under the observation degradation model in (\ref{equation:1}), the following theorem establishes that its two branches recover
orthogonal identifiable components and jointly yield the closest reconstruction
to $\mathcal{X}_0$ satisfying both observations, without requiring rank
assumptions on the degradation operators.
\begin{table*}[!t]
  \centering
  \caption{Quantitative metrics of the compared approaches on the Pavia dataset}
  \label{tab:Pavia_results}
  {\small
  \setlength{\tabcolsep}{4pt}
  \renewcommand{\arraystretch}{1.05}
  \begin{tabular}{*{13}{c}}
    \toprule
\multirow[c]{2}{*}{\raisebox{-1.0ex}{Dataset}}
& \multirow[c]{2}{*}{\raisebox{-1.0ex}{Index}}
    & \multicolumn{2}{c}{Low Rank}
    & \multicolumn{3}{c}{Deep Learning}
    & \multicolumn{6}{c}{Hybrid} \\
    \cmidrule(lr){3-4}
    \cmidrule(lr){5-7}
    \cmidrule(lr){8-13}
    && CTDF & GTNN & CAFE+ & EDIP & CS2DIPs & DELTA & CNN-FUS & LRTFR & CLoRF & SSLRDN & TSR-ITNR \\
    \midrule
    \multirow[c]{4}{*}{Pavia Centre}
      & PSNR$\uparrow$
      & 44.85 & 45.24 & 42.65& 44.57 & \underline{45.88} & 41.22 & 44.90 & 40.63 & 44.45 & 45.41 & \textbf{46.36} \\

      & SSIM$\uparrow$
      & 0.983 & \underline{0.984} & 0.975 & 0.982 & \textbf{0.985} & 0.971 
      & \underline{0.984} & 0.969 & 0.983
      & \textbf{0.985} & \textbf{0.985} \\

      & SAM$\downarrow$
      & 0.072 & 0.064 & 0.078 & 0.064 & \underline{0.060} & 0.093 & 0.065 & 0.104 
      & 0.065 & 0.061 & \textbf{0.058} \\

      & ERGAS$\downarrow$
      & 38.25 & 34.63 & 42.94 & 36.07 & \underline{32.81} & 51.08 & 35.86 & 52.08
      & 36.18 & 33.94 & \textbf{32.55} \\
    \bottomrule
  \end{tabular}}

  \par\vspace{0.3\baselineskip}
  \centering
  \caption{Quantitative metrics of the compared approaches on the CAVE dataset}
  \label{tab:CAVE_results}
  {\small
  \setlength{\tabcolsep}{4.5pt}
  \renewcommand{\arraystretch}{1.05}
  \begin{tabular}{*{13}{c}}
    \toprule
\multirow[c]{2}{*}{\raisebox{-1.0ex}{Dataset}}& \multirow[c]{2}{*}{\raisebox{-1.0ex}{Index}}
    & \multicolumn{2}{c}{Low Rank}
    & \multicolumn{3}{c}{Deep Learning}
    & \multicolumn{6}{c}{Hybrid} \\
    \cmidrule(lr){3-4}
    \cmidrule(lr){5-7}
    \cmidrule(lr){8-13}
    && CTDF & GTNN & CAFE+ & EDIP & CS2DIPs & DELTA & CNN-FUS & LRTFR & CLoRF & SSLRDN & TSR-ITNR \\
    \midrule
        \multirow[c]{4}{*}{Balloons}
      & PSNR$\uparrow$
      & 46.78 & 47.50 & 45.29 & 47.88 & 44.65
      & 43.52 & 43.41 & 44.64 & 47.21
      & \underline{49.25} & \textbf{49.80} \\

      & SSIM$\uparrow$
      & 0.994 & \underline{0.995} & 0.986 & \underline{0.995} & 0.992
      & 0.981 & 0.989 & 0.989 & \underline{0.995}
      & \textbf{0.996} & \textbf{0.996} \\

      & SAM$\downarrow$
      & 0.043 & 0.042 & 0.061 & 0.038 & 0.068
      & 0.106 & 0.067 & 0.058 & 0.040
      & \underline{0.036} & \textbf{0.033} \\

      & ERGAS$\downarrow$
      & 24.36 & 19.81 & 25.89 & 18.01 & 27.07
      & 35.88 & 39.62 & 25.80 & \textbf{15.66}
      & 20.43 & \underline{15.73} \\

    \midrule
    \multirow[c]{4}{*}{Peppers}
      & PSNR$\uparrow$
      & 46.02 & 45.12 & 44.09 & 44.07 & 43.28
      & 40.74 & 45.01 & 44.85 & 46.49
      & \underline{46.89} & \textbf{47.77} \\

      & SSIM$\uparrow$
      & 0.990 & 0.991 & 0.985 & \underline{0.995} & 0.982
      & 0.957 & 0.984 & 0.988 & 0.994
      & \underline{0.995} & \textbf{0.996} \\

      & SAM$\downarrow$
      & 0.077 & 0.080 & 0.093 & \underline{0.049} & 0.122
      & 0.257 & 0.119 & 0.095 & 0.064
      & 0.058 & \textbf{0.047} \\

      & ERGAS$\downarrow$
      & 41.47 & 40.89 & 43.73 & 44.44 & 50.95
      & 66.96 & 44.76 & 40.85 & 38.62
      & \underline{32.30} & \textbf{31.46} \\

    \midrule
    \multirow[c]{4}{*}{Lemons}
      & PSNR$\uparrow$
      & 48.98 & 49.61 & 47.56 & 49.86 & 48.14
      & 45.59 & 45.66 & 47.89 & 49.33
      & \underline{51.94} & \textbf{52.69} \\

      & SSIM$\uparrow$
      & 0.994 & 0.995 & 0.990 & 0.995 & 0.992
      & 0.967 & 0.971 & 0.992 & 0.994
      & \underline{0.996} & \textbf{0.997} \\

      & SAM$\downarrow$
      & 0.047 & 0.053 & 0.072 & 0.044 & 0.072
      & 0.231 & 0.158 & 0.061 & 0.049
      & \underline{0.039} & \textbf{0.034} \\

      & ERGAS$\downarrow$
      & 28.08 & 23.20 & 26.60 & 20.74 & 28.73
      & 52.33 & 42.52 & 25.75 & 23.67
      & \underline{17.18} & \textbf{16.35} \\

    \midrule
    \multirow[c]{4}{*}{Sponges}
      & PSNR$\uparrow$
      & 43.01 & 45.01 & 43.33 & 44.14 & 41.70
      & 41.65 & 40.29 & 43.35 & 45.16
      & \underline{45.48} & \textbf{46.90} \\

      & SSIM$\uparrow$
      & 0.986 & 0.989 & 0.980 & 0.990 & 0.984
      & 0.957 & 0.968 & 0.985 & 0.989
      & \underline{0.992} & \textbf{0.993} \\

      & SAM$\downarrow$
      & 0.049 & 0.034 & 0.050 & \underline{0.030} & 0.105
      & 0.296 & 0.088 & 0.038 & 0.033
      & 0.033 & \textbf{0.027} \\

      & ERGAS$\downarrow$
      & 28.72 & 22.57 & 27.26 & 22.02 & 36.92
      & 63.39 & 41.30 & 24.67 & 20.92
      & \underline{19.41} & \textbf{18.38} \\

    \midrule
    \multirow[c]{4}{*}{Clay}
      & PSNR$\uparrow$
      & 46.64 & 48.40 & 46.80 & 47.61 & 43.67
      & 41.71 & 44.53 & 47.09 & 47.74
      & \underline{49.30} & \textbf{50.53} \\

      & SSIM$\uparrow$
      & 0.987 & 0.993 & 0.987 & 0.992 & 0.949
      & 0.961 & 0.970 & 0.989 & 0.992
      & \underline{0.995} & \textbf{0.996} \\

      & SAM$\downarrow$
      & 0.109 & 0.105 & 0.113 & 0.085 & 0.434
      & 0.220 & 0.202 & 0.110 & 0.092
      & \underline{0.080} & \textbf{0.068} \\

      & ERGAS$\downarrow$
      & 36.67 & 29.68 & 34.12 & 30.62 & 76.61
      & 57.02 & 47.98 & 32.12 & 30.34
      & \underline{25.34} & \textbf{23.52} \\
    \bottomrule
  \end{tabular}
}
\end{table*}
\begin{theorem}[Orthogonal Complementarity and Minimality]\label{property of COGC}
\begin{itshape}
\leavevmode\par\indent
Set $\mathcal{E}_0=\mathcal{X}-\mathcal{X}_0$ and
$\mathcal{E}_{0,s}=\mathcal{E}_0\times_1\boldsymbol{\Pi}_1
\times_2\boldsymbol{\Pi}_2$.
No rank assumption on $\mathbf{P}_1$, $\mathbf{P}_2$, or $\mathbf{P}_3$ is
required. Then:
\begin{enumerate}
    \item[(i)]
    The two branches recover the exact identifiable components
    \begin{equation}
        \Delta\mathcal{X}_Z
        =
        \mathcal{E}_0\times_3\boldsymbol{\Pi}_3,
        \qquad
        \Delta\mathcal{X}_Y
        =
        \mathcal{E}_{0,s}\times_3\mathbf{N}_3.
        \label{eq:cogc-T12}
    \end{equation}

    \item[(ii)]
    The spectral fibers of $\Delta\mathcal{X}_Z$ lie in
    $\operatorname{range}(\mathbf{P}_3^{\top})$, whereas those of
    $\Delta\mathcal{X}_Y$ lie in $\ker(\mathbf{P}_3)$. Consequently,
    \begin{equation}
        \left\langle
            \Delta\mathcal{X}_Y,
            \Delta\mathcal{X}_Z
        \right\rangle_F
        =0,
        \qquad
        \Delta\mathcal{X}_Y\times_3\mathbf{P}_3
        =0.
        \label{eq:cogc-T13}
    \end{equation}
    The orthogonality is retained after spatial degradation:
    \[
        \left\langle
            \Delta\mathcal{X}_Y\times_1\mathbf{P}_1\times_2\mathbf{P}_2,
            \Delta\mathcal{X}_Z\times_1\mathbf{P}_1\times_2\mathbf{P}_2
        \right\rangle_F
        =0.
    \]
    Hence the LR-HSI correction cannot alter the HR-MSI calibration.

    \item[(iii)]
    The remaining error is
    \begin{equation}
        \mathcal{E}_{\mathrm{rem}}
        :=
        \mathcal{X}-\widehat{\mathcal{X}}
        =
        (\mathcal{E}_0-\mathcal{E}_{0,s})\times_3\mathbf{N}_3,
        \label{eq:cogc-T14}
    \end{equation}
    and is invisible to both sensors:
    \[
        \mathcal{E}_{\mathrm{rem}}\times_1\mathbf{P}_1\times_2\mathbf{P}_2
        =0,
        \qquad
        \mathcal{E}_{\mathrm{rem}}\times_3\mathbf{P}_3
        =0.
    \]
    Moreover, the decomposition is Pythagorean:
    \begin{equation}
        \|\mathcal{E}_0\|_F^2
        =
        \|\Delta\mathcal{X}_Z\|_F^2
        +
        \|\Delta\mathcal{X}_Y\|_F^2
        +
        \|\mathcal{E}_{\mathrm{rem}}\|_F^2.
        \label{eq:cogc-T15}
    \end{equation}

    \item[(iv)]
    The estimate $\widehat{\mathcal{X}}$ satisfies both observation equations
    exactly and is the unique minimum-change feasible reconstruction:
    \begin{equation}
        \begin{aligned}
            \widehat{\mathcal{X}}
            ={}&
            \underset{\widetilde{\mathcal{X}}}{\operatorname{argmin}}
            \left\|\widetilde{\mathcal{X}}-\mathcal{X}_0\right\|_F^2
            \\
            \text{s.t.}\quad
            &\widetilde{\mathcal{X}}
            \times_1\mathbf{P}_1\times_2\mathbf{P}_2
            =\mathcal{Y},\\
            &\widetilde{\mathcal{X}}\times_3\mathbf{P}_3
            =\mathcal{Z}.
        \end{aligned}
        \label{eq:cogc-T16}
    \end{equation}
\end{enumerate}
\end{itshape}
\end{theorem}

Theorem~\ref{property of COGC} gives an error-decomposition interpretation of
COGC. The HR-MSI and LR-HSI branches recover sensor-identifiable components
of the Stage-1 error in orthogonal spectral subspaces. Since the LR-HSI
correction lies in the null space of the spectral response, it cannot disturb
the HR-MSI calibration. After correction, the remaining error is invisible
to both sensors and forms a Pythagorean decomposition with the recovered
components. The calibrated result also satisfies both observation equations
and is the unique feasible reconstruction closest to $\mathcal{X}_0$. Thus,
COGC exploits observable residual information through complementary,
noninterfering corrections while minimally modifying the preliminary
estimate.

\setcounter{topnumber}{4}
\setcounter{dbltopnumber}{4}
\setcounter{totalnumber}{8}
\renewcommand{\topfraction}{0.99}
\renewcommand{\dbltopfraction}{0.99}
\renewcommand{\textfraction}{0.01}

\makeatletter
\setlength{\@fptop}{0pt}
\setlength{\@fpbot}{0pt plus 1fil}
\setlength{\@dblfptop}{0pt}
\setlength{\@dblfpbot}{0pt plus 1fil}
\makeatother


\begin{table*}[!t]
  \centering
  \caption{Quantitative metrics of the compared approaches on the ICVL dataset}
  \label{tab:ICVL_results}
  {\small
    \setlength{\tabcolsep}{4.5pt}
    \renewcommand{\arraystretch}{1.05}
    \begin{tabular}{*{13}{c}}
      \toprule
      \multirow[c]{2}{*}{\raisebox{-1.0ex}{Dataset}}
        & \multirow[c]{2}{*}{\raisebox{-1.0ex}{Index}}
        & \multicolumn{2}{c}{Low Rank}
        & \multicolumn{3}{c}{Deep Learning}
        & \multicolumn{6}{c}{Hybrid} \\
      \cmidrule(lr){3-4}
      \cmidrule(lr){5-7}
      \cmidrule(lr){8-13}
        & & CTDF & GTNN & CAFE+ & EDIP & CS2DIPs
        & DELTA & CNN-FUS & LRTFR & CLoRF & SSLRDN & TSR-ITNR \\
      \midrule

      \multirow[c]{4}{*}{BGU}
        & PSNR$\uparrow$
        & 50.30 & 50.15 & 51.17 & 54.41 & 51.51
        & 48.62 & 46.76 & 50.96 & 53.53
        & \underline{54.77} & \textbf{56.61} \\

        & SSIM$\uparrow$
        & 0.996 & 0.997 & 0.996
        & \underline{0.998} & \underline{0.998}
        & 0.993 & 0.987 & 0.997
        & \underline{0.998}
        & \textbf{0.999} & \textbf{0.999} \\

        & SAM$\downarrow$
        & 0.019 & 0.020 & 0.020 & \underline{0.011} & 0.018
        & 0.030 & 0.041 & 0.020 & 0.014
        & 0.012 & \textbf{0.010} \\

        & ERGAS$\downarrow$
        & 21.29 & 22.10 & 13.97 & 9.65 & 16.16
        & 24.02 & 26.20 & 14.17 & 10.62
        & \underline{8.90} & \textbf{7.66} \\

      \midrule
      \multirow[c]{4}{*}{Flower}
        & PSNR$\uparrow$
        & 48.09 & 48.46 & 46.40 & 48.88 & 49.02
        & 44.89 & 45.31 & 44.69
        & \underline{49.38} & 48.81 & \textbf{50.23} \\

        & SSIM$\uparrow$
        & \underline{0.995} & \underline{0.995}
        & 0.989 & \underline{0.995}
        & \textbf{0.996} & 0.987 & 0.990 & 0.987
        & \textbf{0.996} & \textbf{0.996} & \textbf{0.996} \\

        & SAM$\downarrow$
        & 0.012 & 0.011 & 0.016 & 0.011 & 0.011
        & 0.018 & 0.017 & 0.014
        & \underline{0.010} & 0.011 & \textbf{0.009} \\

        & ERGAS$\downarrow$
        & 15.25 & 12.47 & 14.42 & 10.59 & 11.74
        & 21.43 & 19.08 & 16.43
        & \textbf{9.81} & 10.89 & \underline{10.05} \\

      \midrule
      \multirow[c]{4}{*}{Eve}
        & PSNR$\uparrow$
        & 47.94 & 49.50 & 48.52 & \underline{51.47} & 48.99
        & 44.45 & 46.55 & 46.79
        & 51.37 & 51.05 & \textbf{51.95} \\

        & SSIM$\uparrow$
        & 0.994 & 0.995 & 0.993
        & \underline{0.997} & 0.996
        & 0.986 & 0.993 & 0.992
        & \underline{0.997}
        & \textbf{0.998} & \textbf{0.998} \\

        & SAM$\downarrow$
        & 0.017 & 0.014 & 0.017
        & \underline{0.011} & 0.016
        & 0.030 & 0.020 & 0.016
        & \underline{0.011} & \underline{0.011}
        & \textbf{0.010} \\

        & ERGAS$\downarrow$
        & 20.98 & 12.95 & 13.22
        & \underline{9.65} & 16.37
        & 30.88 & 20.56 & 15.53
        & 9.87 & 10.14 & \textbf{9.60} \\

      \midrule
      \multirow[c]{4}{*}{Hill}
        & PSNR$\uparrow$
        & 52.57 & 52.52 & 51.46 & 53.02 & 52.60
        & 48.35 & 46.54 & 49.69 & 53.13
        & \underline{54.58} & \textbf{55.73} \\

        & SSIM$\uparrow$
        & \underline{0.997} & 0.996 & 0.995
        & \underline{0.997} & \underline{0.997}
        & 0.992 & 0.985 & 0.994
        & \underline{0.997}
        & \textbf{0.998} & \textbf{0.998} \\

        & SAM$\downarrow$
        & 0.010 & 0.009 & 0.011
        & \underline{0.008} & 0.009
        & 0.015 & 0.022 & 0.009
        & \underline{0.008} & \underline{0.008}
        & \textbf{0.007} \\

        & ERGAS$\downarrow$
        & 11.66 & 9.72 & 9.75 & 8.23 & 9.14
        & 16.59 & 18.83 & 11.42 & 7.82
        & \underline{6.71} & \textbf{6.41} \\

      \midrule
      \multirow[c]{4}{*}{Nachal}
        & PSNR$\uparrow$
        & 51.54 & 51.09 & 49.41 & 50.00 & 50.88
        & 47.22 & 46.94 & 48.65 & 53.26
        & \underline{54.47} & \textbf{55.36} \\

        & SSIM$\uparrow$
        & 0.996 & 0.996 & 0.993
        & \underline{0.998} & 0.997
        & 0.992 & 0.991 & 0.993 & 0.997
        & \textbf{0.999} & \textbf{0.999} \\

        & SAM$\downarrow$
        & 0.015 & 0.016 & 0.019
        & \textbf{0.009} & 0.013
        & 0.025 & 0.028 & 0.016
        & 0.012
        &\underline{0.010} & \textbf{0.009} \\

        & ERGAS$\downarrow$
        & 15.96 & 12.55 & 13.59 & 13.78 & 23.97
        & 20.71 & 21.00 & 14.25 & 8.58
        & \underline{7.43} & \textbf{7.29} \\
      \bottomrule
    \end{tabular}
  }
\end{table*}


\section{Experiments and Analysis}
\label{sec:experiments}

To comprehensively evaluate the proposed method, we conduct comparative
experiments, ablation studies, and evaluations on downstream semantic
segmentation. Experimental results on simulated datasets are presented to
validate the effectiveness of the proposed method. For all experiments, the
hyperparameters are empirically set to $\lambda=0.5$, $\beta=0.05$, and
$\mu=1.05$. Optimization is performed using the Adam optimizer with a learning
rate of $l_r=5\times10^{-5}$.
\begin{figure*}[!t]
    \centering
    \setlength{\abovecaptionskip}{2pt}

    \makebox[\textwidth][l]{%
    \hspace*{-0.03\textwidth}
    \includegraphics[
        width=1.04\textwidth,
        trim={0mm 2mm 0mm 0mm},
        clip
    ]{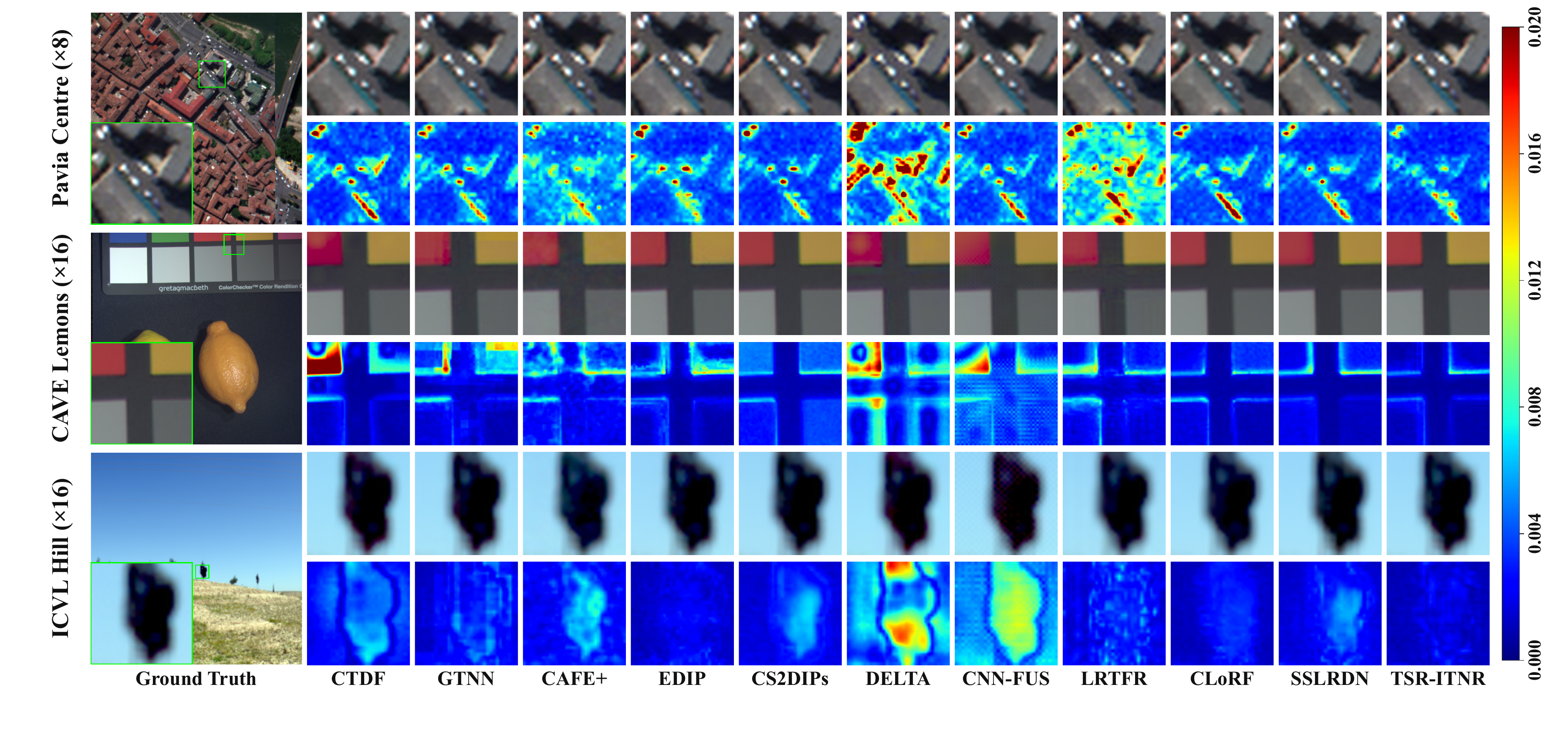}%
}%
    \vspace{-8pt}
    \caption{Visual comparisons of selected representative methods on
    Pavia Centre ($\times 8$), CAVE Lemons ($\times 16$), and ICVL Hill
    ($\times 16$). The three large images present pseudocolor outputs for
    representative samples, with green boxes indicating regions selected
    for local magnification. For each dataset block, the top row shows
    enlarged pseudocolor details, and the bottom row shows the corresponding
    residual maps.}
    \label{fig:combined_hsi_comparison}
\end{figure*}
\subsection{Dataset Description}

\subsubsection{Pavia Dataset}

The Pavia dataset consists of two hyperspectral scenes, namely Pavia Centre and Pavia University, of which Pavia Centre is used in this study. The Pavia Centre scene has a spatial size of $1096 \times 1096$ pixels and originally contains 115 spectral bands. After removing 13 low-SNR bands, a $256 \times 256$ patch with 102 spectral bands is selected from the upper-left region as the reference HR-HSI. To obtain the LR-HSI, the reference HR-HSI is blurred using an $8 \times 8$ Gaussian filter and subsequently downsampled by a factor of 8. The corresponding HR-MSI is simulated using an IKONOS-like reflectance spectral response filter.

\subsubsection{CAVE Dataset}

The CAVE dataset consists of 32 hyperspectral images, each with a spatial
resolution of $512\times512$ pixels and 31 spectral bands. Five scenes,
namely Balloons, Peppers, Lemons, Sponges, and Clay, are randomly selected and
treated as the reference HR-HSIs. For each scene, the LR-HSI is produced by
applying an $8\times8$ Gaussian blur kernel to the HR-HSI, followed by spatial
downsampling with a scale factor of 16. The HR-MSI is synthesized through
spectral degradation using the spectral response function of a Nikon D700
camera~\cite{dian2019learning}.

\subsubsection{ICVL Dataset}

The ICVL dataset comprises 200 hyperspectral images covering diverse indoor and
outdoor real-world scenes. Each image contains 31 spectral bands spanning
400-700~nm and has a spatial resolution of $1390\times1300$ pixels. We randomly
select five scenes, namely BGU\_0522-1217, Flower\_0325-1336, Eve\_0331-1551,
Hill\_0325-1228, and Nachal\_0823-1213, and crop the central $512\times512$
region of each scene as the reference HR-HSI. For brevity, these five scenes are
hereafter referred to as BGU, Flower, Eve, Hill, and Nachal, respectively. The
LR-HSI and HR-MSI are simulated using the same degradation settings as those
adopted for the CAVE dataset.

\subsection{Compared Methods and Quantitative Metrics}

To assess the effectiveness of TSR-ITNR, we conduct a comprehensive comparison with ten representative methods, including two low-rank representation-based methods: CTDF~\cite{xu2024coupled} and GTNN~\cite{dian2024hyperspectral}; three deep learning-based methods: CAFE+~\cite{ke2026content}, EDIP~\cite{li2025enhanced}, and CS2DIPs~\cite{fang2024cs2dips}; and five hybrid methods: DELTA~\cite{yang2025delta}, CNN-FUS~\cite{dian2020regularizing}, LRTFR~\cite{luo2023low}, CLoRF~\cite{wang2025hyperspectral}, and SSLRDN~\cite{chen2025learning}. Since CAFE+, DELTA, and LRTFR were not originally designed for HMIF, we make the necessary task-specific adaptations while preserving their core formulations. All model parameters are fine-tuned based on the settings recommended by the original papers. Fusion quality is evaluated using four commonly used metrics: peak signal-to-noise ratio (PSNR), structural similarity index (SSIM), spectral angle mapper (SAM), and relative dimensionless global error in synthesis (ERGAS).
\begin{figure}[!tp]
  \centering
  \includegraphics[width=0.47\textwidth]{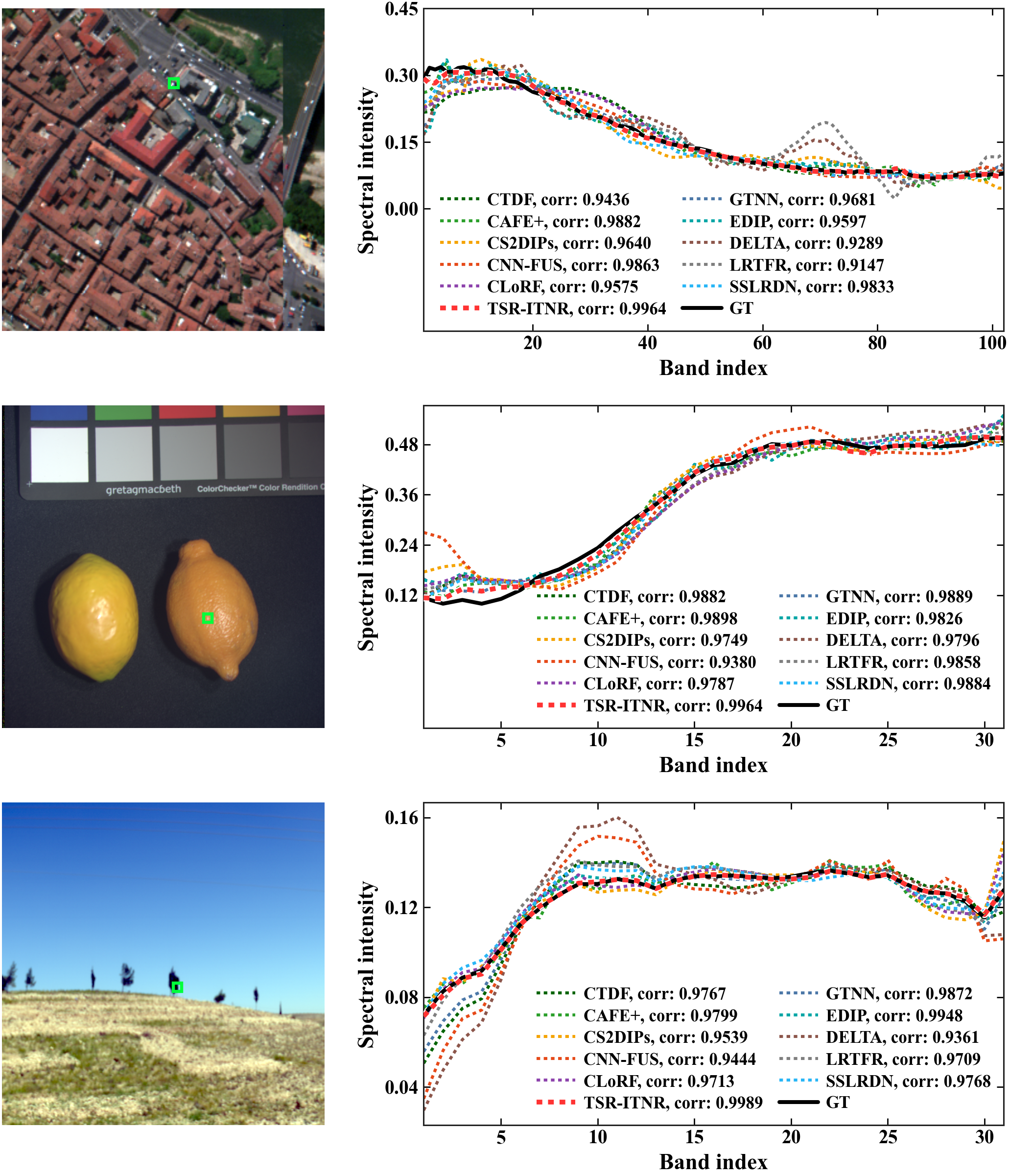}
  \caption{Spectral curves of different methods in three scenes, from top to
  bottom: Pavia Centre, CAVE Lemons, and ICVL Hill.}
  \label{fig:spectral_curves}

  \vspace{0.8\baselineskip}

  \captionof{table}{The Ablation Study Results of the Proposed Method With
  Different Strategies}
  \label{tab:ablation_study}
  {\footnotesize
  \setlength{\tabcolsep}{6pt}
  \renewcommand{\arraystretch}{1.05}
  \begin{tabular}{@{}ccccc@{}}
    \toprule
    Method & PSNR$\uparrow$ & SSIM$\uparrow$ & SAM$\downarrow$
    & ERGAS$\downarrow$ \\
    \midrule
    w/o Tucker & 43.15 & 0.966 & 0.031 & 27.56 \\
    \midrule
    w/o MSCR & 50.30 & 0.995 & 0.016 & 14.13 \\
    \midrule
    w/o AGPSR & \underline{52.99} & \underline{0.997}
    & \underline{0.011} & 9.42 \\
    \midrule
    w/o \mbox{MSCR \& AGPSR} & 49.40 & 0.994 & 0.018 & 15.46 \\
    \midrule
    w/o COGC & 52.64 & \underline{0.997} & \underline{0.011}
    & \underline{9.01} \\
    \midrule
    w/o LR-HSI & 34.49 & 0.983 & 0.107 & 88.15 \\
    \midrule
    w/o HR-MSI & 28.20 & 0.749 & 0.060 & 149.08 \\
    \midrule
    TSR-ITNR & \textbf{53.98} & \textbf{0.998} & \textbf{0.009}
    & \textbf{8.20} \\
    \bottomrule
  \end{tabular}}
  \vspace{-4pt}
\end{figure}
\subsection{Comparative Experimental Results}

Tables~\ref{tab:Pavia_results}--\ref{tab:ICVL_results} present the quantitative
results of the comparative experiments, where the best and second-best results
are highlighted in boldface and underlined, respectively. TSR-ITNR achieves
superior performance across nearly all datasets and evaluation metrics,
demonstrating consistent reconstruction accuracy across scenes with diverse
spatial structures and spectral characteristics. In contrast, the competing methods exhibit more pronounced performance variations across different datasets. CS2DIPs employs two separate
DIPs for spatial and spectral modeling, which limits their interaction and
generally yields lower accuracy than TSR-ITNR. Compared with TSR-ITNR, CLoRF
is less effective in jointly capturing fine-grained spatial structures and
complex spectral information, resulting in lower performance in most cases.
SSLRDN primarily emphasizes spectral subspace decomposition, while providing
a less comprehensive treatment of spatial low-rank structure and complementary
observation information, which limits its fusion accuracy. Overall, these
comparisons highlight the effectiveness of TSR-ITNR in coordinating
spatial--spectral modeling with complementary observation utilization for
high-quality fusion.

Fig.~\ref{fig:combined_hsi_comparison} compares the reconstructed images and
error maps of different methods on the three simulated datasets, with green
boxes marking representative regions. Most methods recover the principal
scene structures, making their differences difficult to distinguish from the
reconstructed images alone. DELTA, however, exhibits noticeable blurring and
loss of fine edge details across all three scenes. The differences are clearer
in the error maps: several competing methods retain pronounced errors around
structural boundaries and textured regions, whereas TSR-ITNR produces
consistently darker maps with fewer localized artifacts, indicating more
accurate spatial reconstruction. Fig.~\ref{fig:spectral_curves} further
compares the spectral curves at representative pixels. TSR-ITNR follows the
ground truth most closely across the spectral range and achieves the highest
correlation coefficient in all three scenes, while the other methods show
larger deviations near sharp transitions and local extrema. These qualitative
results demonstrate the superior spatial and spectral fidelity of TSR-ITNR.
\begin{figure*}
\centering
  \includegraphics[width=1.0\textwidth]{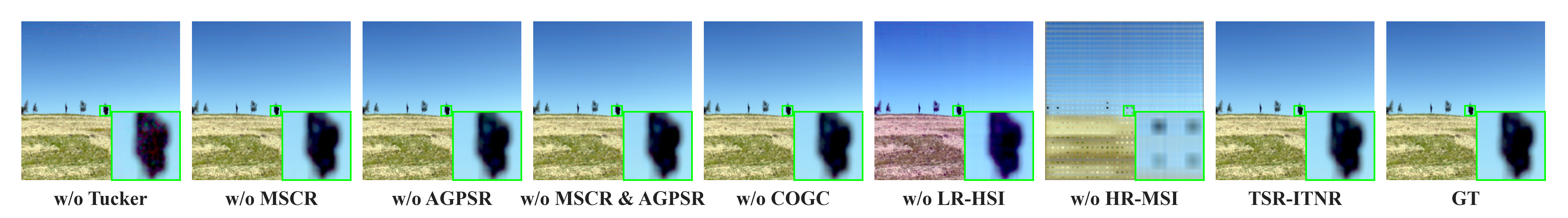}
  \caption{Visual comparison of different ablation settings on ICVL Hill
($\times 16$).}
  \label{fig:Ablation study}
\end{figure*}
\subsection{Ablation Study Results}
To assess the contribution of each component to the overall reconstruction
performance, we conduct a series of ablation studies. Table~\ref{tab:ablation_study}
reports the quantitative results averaged over five simulated scenes from the
ICVL dataset, while Fig.~\ref{fig:Ablation study} presents the corresponding
visual comparison on the ICVL Hill scene.
\subsubsection{Ablation Study on Tucker-Based Representation}
To assess the overall contribution of the Tucker-based structured
reconstruction, we remove the Tucker decomposition, MSCR, and AGPSR and
instead use GFINER to reconstruct the HR-HSI directly from hash-encoded
spatial--spectral coordinates, while retaining COGC. The resulting model is
denoted as w/o Tucker. Its pronounced performance degradation demonstrates
the benefit of structured low-rank spatial--spectral modeling and refinement.
\subsubsection{Ablation Study on MSCR}
To evaluate the effectiveness of MSCR, we remove it from TSR-ITNR while
retaining AGPSR and COGC, denoted as w/o MSCR. The unrefined low-rank spatial
coefficient tensor is directly combined with the spectral basis refined by
AGPSR to form the preliminary HR-HSI. The clear performance degradation
indicates that MSCR effectively captures multiscale spatial dependencies and
fine structural details beyond the Tucker-based low-rank representation.
\subsubsection{Ablation Study on AGPSR}
To evaluate the effectiveness of AGPSR, we remove it from TSR-ITNR while
retaining MSCR and COGC, denoted as w/o AGPSR. The unrefined spectral basis
is directly combined with the low-rank spatial coefficient tensor refined by
MSCR to form the preliminary HR-HSI. Removing AGPSR leads to a modest but
consistent performance decline across all metrics, indicating that its
geometry-aware refinement of inter-band dependencies provides complementary
benefits to HR-HSI reconstruction.
\subsubsection{Ablation Study on MSCR and AGPSR}
To evaluate the effectiveness of joint spatial-spectral refinement, we
remove MSCR and AGPSR while retaining the Tucker decomposition and COGC,
denoted as w/o \mbox{MSCR \& AGPSR}. In this setting, the Tucker-derived
low-rank spatial coefficient tensor and spectral basis are directly combined
without further refinement. This setting outperforms w/o Tucker but
underperforms the complete TSR-ITNR. The comparison with w/o Tucker
demonstrates the effectiveness of Tucker-based structural modeling, while
the gap from the complete model shows that jointly refining both components
enables more effective representation of complex spatial-spectral details.
\subsubsection{Ablation Study on COGC}
To evaluate the effectiveness of COGC, we remove the entire second stage and
directly use the Stage 1 reconstruction as the final output, denoted as w/o
COGC. In Stage 1, the LR-HSI and HR-MSI guide the reconstruction through soft
data-consistency terms, which may not fully exploit their complementary
information. The resulting performance degradation indicates that COGC
further leverages this complementarity to calibrate the preliminary
reconstruction and improve fusion accuracy.

\subsubsection{Ablation Study on the Effectiveness of Fusion} To evaluate the effectiveness of fusing the two observations, we reconstruct
the HR-HSI using either the HR-MSI alone (w/o LR-HSI) or the LR-HSI alone
(w/o HR-MSI). Both settings exhibit substantial performance degradation,
demonstrating the benefit of jointly exploiting the two observations. This
benefit arises from their complementary information: the HR-MSI provides fine
spatial details, whereas the LR-HSI preserves rich spectral information.
\begin{figure*}[!t]
\centering
\includegraphics[width=1.0\textwidth]{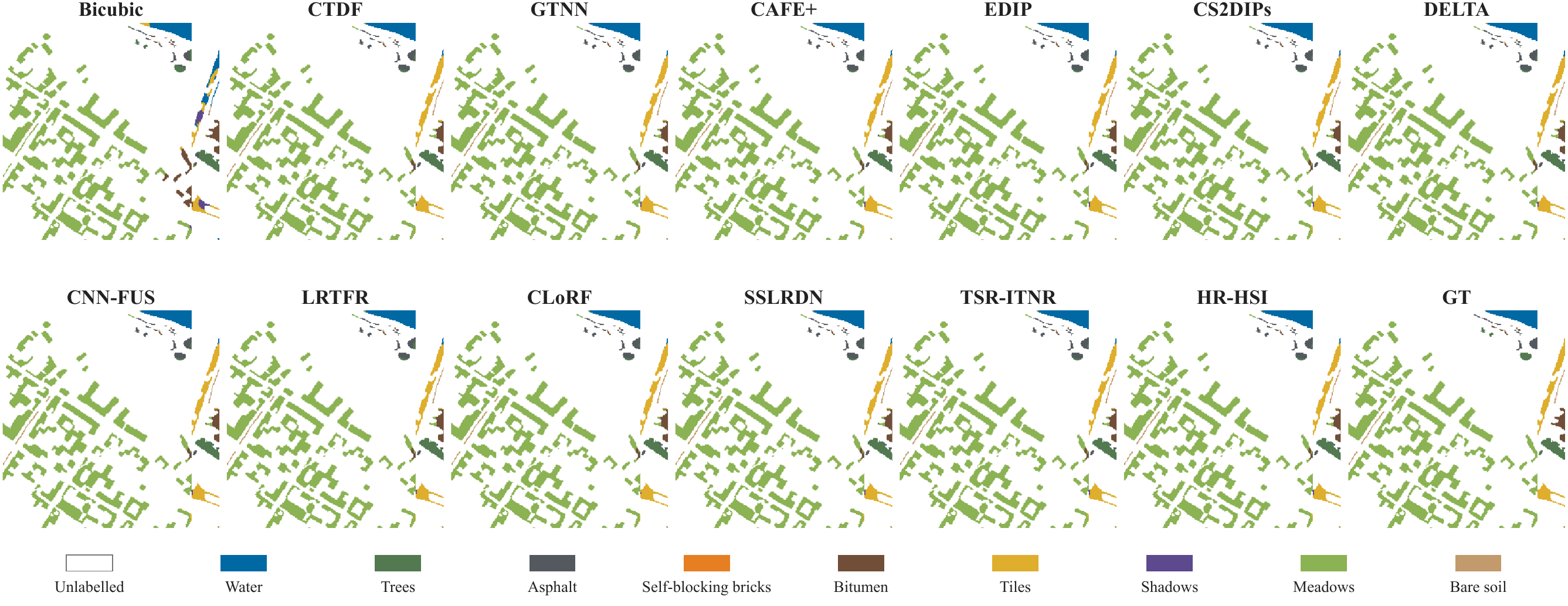}
\captionof{figure}{Visual comparison of semantic segmentation maps obtained from the
Bicubic baseline, different reconstructed HSIs, and the reference HR-HSI on
the Pavia Centre test region. GT denotes the ground-truth labels.}
\label{fig:pavia_segmentation_maps}

\par\vspace{0.4\baselineskip}

\captionof{table}{Quantitative results of semantic segmentation on
the Pavia Centre}
\label{tab:pavia_segmentation_results}
{\small
\setlength{\tabcolsep}{2.5pt}
\renewcommand{\arraystretch}{1.0}
\begin{tabular}{*{14}{c}}
\toprule
\multirow[c]{2}{*}{\raisebox{-1.0ex}{Index}}
& \multicolumn{1}{c}{Baseline}
& \multicolumn{2}{c}{Low Rank}
& \multicolumn{3}{c}{Deep Learning}
& \multicolumn{6}{c}{Hybrid}
& \multicolumn{1}{c}{Reference} \\
\cmidrule(lr){2-2}
\cmidrule(lr){3-4}
\cmidrule(lr){5-7}
\cmidrule(lr){8-13}
\cmidrule(lr){14-14}
& Bicubic & CTDF & GTNN & CAFE+ & EDIP & CS2DIPs
& DELTA & CNN-FUS & LRTFR & CLoRF & SSLRDN & TSR-ITNR & HR-HSI \\
\midrule
mIoU$\uparrow$
& 55.36 & 81.98 & 82.61 & 82.64 & 82.50 & \underline{82.71}
& 82.23 & 82.49 & 81.96 & 82.16 & 82.50 & \textbf{82.77}
& \textit{82.89} \\
Macro-F1$\uparrow$
& 66.83 & 88.67 & 89.15 & 89.12 & 89.03 &  \underline{89.21}
& 88.82 & 89.02 & 88.61 & 88.79 & 89.08 & \textbf{89.23}
& \textit{89.33} \\
OA$\uparrow$
& 90.93 & 97.58 & \textbf{97.66} & 97.63 & 97.63 & \underline{97.65}
& 97.58 & 97.60 & 97.55 & 97.60 & \underline{97.65} & \textbf{97.66}
& \textit{97.78} \\
AA$\uparrow$
& 71.43 & 90.07 & 90.74 & 90.77 & 90.56 & \underline{90.85}
& 90.09 & 90.68 & 89.81 & 90.38 & 90.81 & \textbf{90.90}
& \textit{90.91} \\
\bottomrule
\end{tabular}}
\end{figure*}
\subsection{Semantic Segmentation Results}

Since conventional reconstruction metrics, such as PSNR and SSIM, do not
fully reflect the utility of reconstructed HSIs in downstream
applications~\cite{10683775}, we further assess their semantic segmentation
performance. Following the data partition used in the preceding
super-resolution experiments, the upper-left $256\times256$ region of the
Pavia Centre scene is reserved for testing, while the remaining area is
divided into spatially disjoint training and validation sets. We adapt
UNetFormer~\cite{wang2022unetformer} to directly process hyperspectral inputs
with 102 spectral bands. The model is trained only once on the reference
HR-HSI, with model selection performed on the validation set. Its parameters
are then frozen, and the same model is applied to the bicubic baseline, all
reconstructed HSIs, and the reference HR-HSI without method-specific
fine-tuning. Segmentation performance is evaluated using mean intersection
over union (mIoU), macro-averaged F1 score (Macro-F1), overall accuracy (OA),
and average accuracy (AA). Table~\ref{tab:pavia_segmentation_results} reports the quantitative results,
where the best and second-best results among the reconstructed HSIs are
highlighted in boldface and underlined, respectively. The results obtained
from the reference HR-HSI in the rightmost column are italicized for
distinction. The corresponding segmentation maps are shown in
Fig.~\ref{fig:pavia_segmentation_maps}.

\section{Conclusion}
\label{sec:conclusion}

In this article, we presented TSR-ITNR, a self-supervised two-stage framework
for HMIF. Stage 1 employs an implicit Tucker representation to
preserve the high-order structure and capture the global low-rank
spatial-spectral correlations of the HR-HSI. Dedicated spatial and spectral
refinements further enhance fine spatial details and interband dependencies,
while PnP-HQS incorporates a pretrained denoiser as a deep prior. Stage 2 applies the parameter-free COGC to calibrate the preliminary
reconstruction using complementary corrections derived
from the LR-HSI and HR-MSI. Theoretical analysis establishes the basis
independence and Gram preservation of spectral refinement, together with the
orthogonal complementarity and minimum-change property of COGC. Experiments on simulated datasets demonstrate strong quantitative,
visual, and spectral reconstruction performance. Ablation studies confirm the
contribution of the proposed components, while semantic segmentation further
validates the downstream utility of the reconstructed HSIs. Future work will
extend the proposed implicit tensor representation and complementary
calibration framework to other hyperspectral restoration and multimodal
remote-sensing fusion tasks.
\nocite{*}
\bibliographystyle{IEEEtran}

\appendix
This appendix provides detailed proofs of the two theoretical results presented
in the main text. Theorem~\ref{thm:agpsr} establishes the validity, basis
independence, and Gram-preserving property of the AGPSR update.
Theorem~\ref{property of COGC} characterizes the orthogonal complementarity of
the two COGC branches, the sensor-unobservable residual, and the observation
consistency and minimum-change property of the calibrated reconstruction.
\par\medskip
\begin{proof}[Proof of Theorem~\ref{thm:agpsr}]
We establish the four claims in sequence. First, since
\(\mathbf{Q}^{\top}\mathbf{Q}=\mathbf{I}_{r_3}\) and
\(\boldsymbol{\Pi}_{\mathbf{W}}=\mathbf{Q}\mathbf{Q}^{\top}\),
\begin{align*}
    \mathbf{Q}^{\top}\mathbf{T}
    &=
    \mathbf{Q}^{\top}
    \bigl(\mathbf{I}_C-\mathbf{Q}\mathbf{Q}^{\top}\bigr)\mathbf{S}=
    \mathbf{Q}^{\top}\mathbf{S}
    -(\mathbf{Q}^{\top}\mathbf{Q})\mathbf{Q}^{\top}\mathbf{S}\\
    &=
    \mathbf{Q}^{\top}\mathbf{S}
    -\mathbf{I}_{r_3}\mathbf{Q}^{\top}\mathbf{S}
    =\mathbf{0}.
\end{align*}
Taking transposes gives
\(\mathbf{T}^{\top}\mathbf{Q}=\mathbf{0}\). The horizontal representation of
the tangent space of the Grassmann manifold consists precisely of matrices
\(\mathbf{Z}\) satisfying
\(\mathbf{Q}^{\top}\mathbf{Z}=\mathbf{0}\); hence \(\mathbf{T}\) is a valid
horizontal tangent direction. Moreover,
\begin{align*}
    \mathbf{K}\mathbf{Q}
    &=
    \bigl(
        \mathbf{T}\mathbf{Q}^{\top}
        -\mathbf{Q}\mathbf{T}^{\top}
    \bigr)\mathbf{Q}=
    \mathbf{T}(\mathbf{Q}^{\top}\mathbf{Q})
    -\mathbf{Q}(\mathbf{T}^{\top}\mathbf{Q})
    =\mathbf{T}.
\end{align*}
This proves part (i).

\smallskip
\noindent Next, let
\(\widetilde{\mathbf{Q}}=\mathbf{Q}\mathbf{O}\) and
\(\widetilde{\mathbf{R}}=\mathbf{O}^{\top}\mathbf{R}\), where
\(\mathbf{O}^{\top}\mathbf{O}=\mathbf{I}_{r_3}\). Then
\[
    \widetilde{\mathbf{Q}}\widetilde{\mathbf{R}}
    =
    \mathbf{Q}\mathbf{O}\mathbf{O}^{\top}\mathbf{R}
    =
    \mathbf{W},
    \qquad
    \widetilde{\mathbf{Q}}^{\top}\widetilde{\mathbf{Q}}
    =
    \mathbf{I}_{r_3}.
\]
Thus this is another orthonormal-basis representation of the same matrix
\(\mathbf{W}\); it need not be a triangular QR factorization. Its projector is
\[
    \widetilde{\boldsymbol{\Pi}}_{\mathbf{W}}
    =
    \widetilde{\mathbf{Q}}\widetilde{\mathbf{Q}}^{\top}
    =
    \mathbf{Q}\mathbf{O}\mathbf{O}^{\top}\mathbf{Q}^{\top}
    =
    \boldsymbol{\Pi}_{\mathbf{W}}.
\]
Since the attention matrices depend only on the basis-invariant projector
$\boldsymbol{\Pi}_{\mathbf{W}}$, the transformation
$\widetilde{\mathbf{Q}}=\mathbf{Q}\mathbf{O}$ leaves them unchanged and yields
$\widetilde{\mathbf{S}}=\mathbf{S}\mathbf{O}$. Hence,
\[
    \widetilde{\mathbf{T}}
    =
    \bigl(
        \mathbf{I}_C-\widetilde{\boldsymbol{\Pi}}_{\mathbf{W}}
    \bigr)\widetilde{\mathbf{S}}
    =
    \bigl(\mathbf{I}_C-\boldsymbol{\Pi}_{\mathbf{W}}\bigr)
    \mathbf{S}\mathbf{O}
    =
    \mathbf{T}\mathbf{O}.
\]
Therefore
\begin{align*}
    \widetilde{\mathbf{K}}=
    \widetilde{\mathbf{T}}\widetilde{\mathbf{Q}}^{\top}
    -
    \widetilde{\mathbf{Q}}\widetilde{\mathbf{T}}^{\top}=
    \mathbf{T}\mathbf{O}\mathbf{O}^{\top}\mathbf{Q}^{\top}
    -
    \mathbf{Q}\mathbf{O}\mathbf{O}^{\top}\mathbf{T}^{\top}
    =
    \mathbf{K}.
\end{align*}
Equation \eqref{eq:agpsr-cayley} then gives
\(\widetilde{\mathbf{C}}_{\alpha}=\mathbf{C}_{\alpha}\) and
\(\widetilde{\mathbf{W}}_{\mathrm{ref}}
=\mathbf{C}_{\alpha}\widetilde{\mathbf{Q}}\widetilde{\mathbf{R}}
=\mathbf{C}_{\alpha}\mathbf{W}
=\mathbf{W}_{\mathrm{ref}}\).
This proves part (ii).

\smallskip
\noindent Subsequently, transposing \eqref{eq:agpsr-generator} gives
\[
    \mathbf{K}^{\top}
    =
    \mathbf{Q}\mathbf{T}^{\top}
    -\mathbf{T}\mathbf{Q}^{\top}
    =
    -\mathbf{K},
\]
so \(\mathbf{K}\) is real skew-symmetric. Put
\[
    \tau=\frac{\alpha}{2},
    \qquad
    \mathbf{M}_{-}=\mathbf{I}_C-\tau\mathbf{K},
    \qquad
    \mathbf{M}_{+}=\mathbf{I}_C+\tau\mathbf{K}.
\]
To show that \(\mathbf{M}_{-}\) is nonsingular, suppose
\(\mathbf{M}_{-}\mathbf{x}=\mathbf{0}\). Then
\(\mathbf{x}=\tau\mathbf{K}\mathbf{x}\), and premultiplication by
\(\mathbf{x}^{\top}\) yields
\[
    \|\mathbf{x}\|_2^2
    =
    \tau\mathbf{x}^{\top}\mathbf{K}\mathbf{x}.
\]
Because \(\mathbf{K}^{\top}=-\mathbf{K}\),
\[
    \mathbf{x}^{\top}\mathbf{K}\mathbf{x}
    =
    \bigl(\mathbf{x}^{\top}\mathbf{K}\mathbf{x}\bigr)^{\top}
    =
    \mathbf{x}^{\top}\mathbf{K}^{\top}\mathbf{x}
    =
    -\mathbf{x}^{\top}\mathbf{K}\mathbf{x},
\]
and therefore
\(\mathbf{x}^{\top}\mathbf{K}\mathbf{x}=0\). It follows that
\(\mathbf{x}=\mathbf{0}\), so the square matrix \(\mathbf{M}_{-}\) is
nonsingular. The same argument, with \(\tau\) replaced by \(-\tau\), proves
that \(\mathbf{M}_{+}\) is nonsingular. Thus the Cayley transform exists for
every \(\alpha\in\mathbb{R}\).

Next,
\(\mathbf{M}_{-}^{\top}=\mathbf{M}_{+}\) and
\(\mathbf{M}_{+}^{\top}=\mathbf{M}_{-}\). Starting from
\(\mathbf{C}_{\alpha}=\mathbf{M}_{-}^{-1}\mathbf{M}_{+}\), its transpose is
\begin{align*}
    \mathbf{C}_{\alpha}^{\top}
    &=
    \bigl(\mathbf{M}_{-}^{-1}\mathbf{M}_{+}\bigr)^{\top}=
    \mathbf{M}_{+}^{\top}
    \bigl(\mathbf{M}_{-}^{-1}\bigr)^{\top}\\
    &=
    \mathbf{M}_{-}
    \bigl(\mathbf{M}_{-}^{\top}\bigr)^{-1}=
    \mathbf{M}_{-}\mathbf{M}_{+}^{-1}.
\end{align*}
The third line uses the exact identity
\((\mathbf{M}_{-}^{-1})^{\top}
=(\mathbf{M}_{-}^{\top})^{-1}\); this is the transpose-of-inverse step.
Furthermore,
\[
\mathbf{M}_{-}\mathbf{M}_{+}
=
(\mathbf{I}_C-\tau\mathbf{K})
(\mathbf{I}_C+\tau\mathbf{K})
=
\mathbf{I}_C-\tau^2\mathbf{K}^2
=
\mathbf{M}_{+}\mathbf{M}_{-}.
\]
Hence \(\mathbf{M}_{-}\) commutes with \(\mathbf{M}_{+}^{-1}\), and
\[
    \mathbf{C}_{\alpha}^{\top}
    =
    \mathbf{M}_{-}\mathbf{M}_{+}^{-1}
    =
    \mathbf{M}_{+}^{-1}\mathbf{M}_{-}
    =
    \mathbf{C}_{\alpha}^{-1}.
\]
Consequently,
\(\mathbf{C}_{\alpha}^{\top}\mathbf{C}_{\alpha}=\mathbf{I}_C\), so
\(\mathbf{C}_{\alpha}\) is orthogonal. 
  Using \(\mathbf{W}_{\mathrm{ref}}=\mathbf{C}_{\alpha}\mathbf{W}\) now gives
\begin{align*}
    \mathbf{W}_{\mathrm{ref}}^{\top}\mathbf{W}_{\mathrm{ref}}
    &=
    \mathbf{W}^{\top}
    \mathbf{C}_{\alpha}^{\top}
    \mathbf{C}_{\alpha}\mathbf{W}=
    \mathbf{W}^{\top}\mathbf{W},
\end{align*}
which proves \eqref{eq:agpsr-T5}. Its entries are the pairwise column inner
products. Its eigenvalues are the squared singular values; hence the remaining
rank, norm, and two-norm condition-number claims follow immediately. This
proves part (iii).
\renewcommand{\qedsymbol}{}
\end{proof}

\setcounter{equation}{0}
\renewcommand{\theequation}{B\arabic{equation}}

\begin{proof}[Proof of Theorem \ref{property of COGC}]
We first record the required matrix and tensor identities, and then prove parts
(i)-(iv) separately.

\smallskip
\noindent First, for every matrix $\mathbf{P}_i$, the Moore--Penrose identities imply
\[
    (\mathbf{P}_i^{\dagger}\mathbf{P}_i)^{\top}
    =
    \mathbf{P}_i^{\dagger}\mathbf{P}_i,
    \qquad
    (\mathbf{P}_i^{\dagger}\mathbf{P}_i)^2
    =
    \mathbf{P}_i^{\dagger}\mathbf{P}_i.
\]
Thus each $\boldsymbol{\Pi}_i$ is an orthogonal projector onto
$\operatorname{range}(\mathbf{P}_i^{\top})$, and
$\mathbf{I}-\boldsymbol{\Pi}_i$, with the identity chosen in the
corresponding dimension, projects onto $\ker(\mathbf{P}_i)$. In particular,
\begin{equation}
    \boldsymbol{\Pi}_3\mathbf{N}_3=0,
    \qquad
    \mathbf{P}_3\mathbf{N}_3=0,
    \qquad
    \mathbf{P}_i\boldsymbol{\Pi}_i=\mathbf{P}_i.
    \label{eq:cogc-T17}
\end{equation}
No rank assumption is used in these identities.

For compatible matrices, mode products obey
\begin{equation}
    (\mathcal{T}\times_n\mathbf{A})\times_n\mathbf{B}
    =
    \mathcal{T}\times_n(\mathbf{B}\mathbf{A}),
    \label{eq:cogc-T18}
\end{equation}
and products on distinct modes commute. Therefore all spatial mode-1/mode-2
projections commute with the spectral mode-3 projections. These facts will
justify the rearrangements below.

\smallskip
\noindent Next, from \ref{equation:1} and \eqref{eq:cogc-residuals}, 
\begin{equation}
    \mathcal{E}_Y
    =
    \mathcal{E}_0\times_1\mathbf{P}_1\times_2\mathbf{P}_2,
    \qquad
    \mathcal{E}_Z
    =
    \mathcal{E}_0\times_3\mathbf{P}_3.
    \label{eq:cogc-T19}
\end{equation}
Substituting the second identity into \eqref{eq:cogc-corrections} and applying
\eqref{eq:cogc-T18} yields
\[
    \Delta\mathcal{X}_Z
    =
    (\mathcal{E}_0\times_3\mathbf{P}_3)
    \times_3\mathbf{P}_3^{\dagger}
    =
    \mathcal{E}_0\times_3
    (\mathbf{P}_3^{\dagger}\mathbf{P}_3)
    =
    \mathcal{E}_0\times_3\boldsymbol{\Pi}_3.
\]
Similarly, substituting the first identity into \eqref{eq:cogc-corrections}, commuting
distinct modes, and composing products on the same mode gives
\begin{align*}
    \Delta\mathcal{X}_Y
    &=
    \mathcal{E}_0
    \times_1(\mathbf{P}_1^{\dagger}\mathbf{P}_1)
    \times_2(\mathbf{P}_2^{\dagger}\mathbf{P}_2)
    \times_3\mathbf{N}_3=
    \mathcal{E}_{0,s}\times_3\mathbf{N}_3.
\end{align*}
This proves part (i).

\smallskip
\noindent Moreover, by part (i), every spectral fiber of
$\Delta\mathcal{X}_Z$ belongs to
$\operatorname{range}(\boldsymbol{\Pi}_3)
=\operatorname{range}(\mathbf{P}_3^{\top})$, whereas every spectral fiber of
$\Delta\mathcal{X}_Y$ belongs to
$\operatorname{range}(\mathbf{N}_3)=\ker(\mathbf{P}_3)$. These two subspaces
are orthogonal because $\boldsymbol{\Pi}_3\mathbf{N}_3=0$. The Frobenius inner
product is the sum of the Euclidean inner products of corresponding spectral
fibers, so
$\langle\Delta\mathcal{X}_Y,\Delta\mathcal{X}_Z\rangle_F=0$.

Spatial degradation takes linear combinations of spectral fibers but does not
mix their spectral coordinates. Linear combinations of fibers in
$\ker(\mathbf{P}_3)$ remain in that subspace, and the same is true for
$\operatorname{range}(\mathbf{P}_3^{\top})$. Hence the two spatially degraded
tensors are still Frobenius-orthogonal. Finally, by \eqref{eq:cogc-T18} and
\eqref{eq:cogc-T17},
\[
    \Delta\mathcal{X}_Y\times_3\mathbf{P}_3
    =
    \mathcal{E}_{0,s}\times_3(\mathbf{P}_3\mathbf{N}_3)
    =0.
\]
Thus the LR-HSI branch is invisible to the HR-MSI sensor and cannot disturb
its calibration. This proves part (ii).

\smallskip
\noindent Subsequently, since
$\boldsymbol{\Pi}_3+\mathbf{N}_3=\mathbf{I}_C$, part (i) gives
\begin{align*}
    \mathcal{E}_{\mathrm{rem}}
    &=
    \mathcal{E}_0-\Delta\mathcal{X}_Z-\Delta\mathcal{X}_Y=
    \mathcal{E}_0\times_3\mathbf{N}_3
    -\mathcal{E}_{0,s}\times_3\mathbf{N}_3\\
    &=
    (\mathcal{E}_0-\mathcal{E}_{0,s})\times_3\mathbf{N}_3,
\end{align*}
which proves \eqref{eq:cogc-T14}. Its spectral observation is zero because
$\mathbf{P}_3\mathbf{N}_3=0$. Its spatial observation is also zero, because
$\mathbf{P}_i\boldsymbol{\Pi}_i=\mathbf{P}_i$ gives
\begin{align*}
    &(\mathcal{E}_0-\mathcal{E}_{0,s})
    \times_1\mathbf{P}_1\times_2\mathbf{P}_2\\
    &\quad=
    \mathcal{E}_0\times_1\mathbf{P}_1\times_2\mathbf{P}_2
    -
    \mathcal{E}_0
    \times_1(\mathbf{P}_1\boldsymbol{\Pi}_1)
    \times_2(\mathbf{P}_2\boldsymbol{\Pi}_2)
    =0.
\end{align*}
It remains to establish pairwise orthogonality. The component
$\Delta\mathcal{X}_Z$ is orthogonal to both remaining components because it
lies in the spectral row space, whereas both $\Delta\mathcal{X}_Y$ and
$\mathcal{E}_{\mathrm{rem}}$ lie in the spectral null space. To compare the
latter two, set $\mathcal{F}=\mathcal{E}_0\times_3\mathbf{N}_3$. Then
\[
    \Delta\mathcal{X}_Y
    =
    \mathcal{F}\times_1\boldsymbol{\Pi}_1\times_2\boldsymbol{\Pi}_2,
    \qquad
    \mathcal{E}_{\mathrm{rem}}
    =
    \mathcal{F}-\Delta\mathcal{X}_Y.
\]
The map
$\mathcal{T}\mapsto
\mathcal{T}\times_1\boldsymbol{\Pi}_1\times_2\boldsymbol{\Pi}_2$ is an
orthogonal projector: under vectorization its matrix is a Kronecker product
of the symmetric idempotent matrices $\boldsymbol{\Pi}_1$,
$\boldsymbol{\Pi}_2$, and an identity matrix. Therefore its image is
orthogonal to its residual, which proves
$\langle\Delta\mathcal{X}_Y,\mathcal{E}_{\mathrm{rem}}\rangle_F=0$. The norm
identity \eqref{eq:cogc-T15} now follows by applying the Pythagorean theorem to
the three mutually orthogonal components. This proves part (iii).

\smallskip
\noindent Furthermore, for the spectral observation, part (ii) and
$\mathbf{P}_3\boldsymbol{\Pi}_3=\mathbf{P}_3$ yield
\begin{align*}
    \widehat{\mathcal{X}}\times_3\mathbf{P}_3
    &=
    \mathcal{X}_0\times_3\mathbf{P}_3
    +\Delta\mathcal{X}_Z\times_3\mathbf{P}_3\\
    &=
    \mathcal{X}_0\times_3\mathbf{P}_3
    +\mathcal{E}_0\times_3\mathbf{P}_3
    =\mathcal{Z}.
\end{align*}
For the spatial observation, part (i), commutation of distinct modes, and
$\mathbf{P}_i\boldsymbol{\Pi}_i=\mathbf{P}_i$ give
\begin{align*}
    &(\Delta\mathcal{X}_Z+\Delta\mathcal{X}_Y)
    \times_1\mathbf{P}_1\times_2\mathbf{P}_2\\
    &\quad=
    \mathcal{E}_Y\times_3\boldsymbol{\Pi}_3
    +\mathcal{E}_Y\times_3\mathbf{N}_3
    =\mathcal{E}_Y.
\end{align*}
Adding the Stage-1 spatial observation therefore gives
$\widehat{\mathcal{X}}\times_1\mathbf{P}_1\times_2\mathbf{P}_2=\mathcal{Y}$.
Hence $\widehat{\mathcal{X}}$ is feasible for \eqref{eq:cogc-T16}.

\smallskip
\noindent Finally, define the joint null space by
\[
    \mathbb{V}_0
    =
    \left\{
        \mathcal{H}:
        \mathcal{H}\times_1\mathbf{P}_1\times_2\mathbf{P}_2=0,
        \ \mathcal{H}\times_3\mathbf{P}_3=0
    \right\}.
\]
The map $\mathcal{T}\mapsto\mathcal{T}\times_3\mathbf{N}_3$ is the
orthogonal projector onto the spectral null space. The map
\[
    \mathcal{T}
    \mapsto
    \mathcal{T}
    -\mathcal{T}\times_1\boldsymbol{\Pi}_1\times_2\boldsymbol{\Pi}_2
\]
is the orthogonal projector onto the null space of the spatial degradation.
Indeed, after vectorizing each spectral band, the spatial degradation is
represented by $\mathbf{P}_2\otimes\mathbf{P}_1$, whose row-space projector is
\[
    (\mathbf{P}_2\otimes\mathbf{P}_1)^{\dagger}
    (\mathbf{P}_2\otimes\mathbf{P}_1)
    =
    \boldsymbol{\Pi}_2\otimes\boldsymbol{\Pi}_1.
\]
The two null-space projectors act on different modes and commute. The product
of two commuting orthogonal projectors is the orthogonal projector onto the
intersection of their ranges. By \eqref{eq:cogc-T14}, their product maps
$\mathcal{E}_0$ exactly to $\mathcal{E}_{\mathrm{rem}}$. Hence
\begin{equation}
\mathcal{E}_{\mathrm{rem}}
=
\operatorname{proj}_{\mathbb{V}_0}(\mathcal{E}_0),\;
\Delta\mathcal{X}_Y+\Delta\mathcal{X}_Z
=
\mathcal{E}_0-\mathcal{E}_{\mathrm{rem}}
\in\mathbb{V}_0^{\perp}.
\label{eq:cogc-T20}
\end{equation}
Now write any feasible candidate as
$\widetilde{\mathcal{X}}=\mathcal{X}_0+\Delta\mathcal{X}$. Since both
$\widetilde{\mathcal{X}}$ and the true $\mathcal{X}$ satisfy the two
observation equations, $\Delta\mathcal{X}-\mathcal{E}_0\in\mathbb{V}_0$.
Thus $\Delta\mathcal{X}=\mathcal{E}_0+\mathcal{H}$ for some
$\mathcal{H}\in\mathbb{V}_0$. Using \eqref{eq:cogc-T20},
\[
    \Delta\mathcal{X}
    =
    (\Delta\mathcal{X}_Y+\Delta\mathcal{X}_Z)
    +(\mathcal{E}_{\mathrm{rem}}+\mathcal{H}),
\]
where the first term lies in $\mathbb{V}_0^{\perp}$ and the second in
$\mathbb{V}_0$. Therefore
\[
    \|\Delta\mathcal{X}\|_F^2
    =
    \|\Delta\mathcal{X}_Y+\Delta\mathcal{X}_Z\|_F^2
    +
    \|\mathcal{E}_{\mathrm{rem}}+\mathcal{H}\|_F^2.
\]
The right-hand side is minimized uniquely when
$\mathcal{H}=-\mathcal{E}_{\mathrm{rem}}$, in which case
$\Delta\mathcal{X}=\Delta\mathcal{X}_Y+\Delta\mathcal{X}_Z$ and
$\widetilde{\mathcal{X}}=\widehat{\mathcal{X}}$. This proves the
minimum-change property and its uniqueness, completing part (iv) and the
proof.
\renewcommand{\qedsymbol}{}
\end{proof}

\bibliography{refs}
\end{document}